\documentclass[letterpaper,11pt]{article}

\usepackage[utf8]{inputenc}
\usepackage[T1]{fontenc}
\usepackage{microtype}
\usepackage{csquotes}
\usepackage[osf,sc]{mathpazo}
\RequirePackage[scaled=0.90]{helvet}
\RequirePackage[scaled=0.85]{beramono}
\RequirePackage{textcomp}

\usepackage{diagbox}
\usepackage{makecell}
\usepackage{booktabs}
\usepackage{tikz}
\usepackage{pgfplots}
\usepackage{amsmath}
\usepackage{amsfonts}
\usepackage{amssymb}
\usepackage{amsthm}
\usepackage{url,ifthen}
\usepackage{enumitem}
\usepackage{srcltx}
\usepackage{multirow}
\usepackage{boxedminipage}
\usepackage[margin=1in]{geometry}
\usepackage{nicefrac}
\usepackage{xspace}
\usepackage{graphicx}
\usepackage{color}
\usepackage{subfigure}
\usepackage{colortbl}
\usepackage{setspace}
\usepackage{natbib}
\setcitestyle{square, authoryear}

 \pgfplotsset{compat=1.11} 

\usetikzlibrary{shapes,decorations,arrows,calc,arrows.meta,fit,positioning}
\tikzset{
    -Latex,auto,node distance =1 cm and 1 cm,semithick,
    state/.style ={circle, draw, minimum width = 1 cm},
    point/.style = {circle, draw, inner sep=0.04cm,fill,node contents={}},
    bidirected/.style={Latex-Latex,dashed},
    el/.style = {inner sep=2pt, align=left, sloped}
}

\DeclareMathOperator*{\argmin}{arg\,min}

\definecolor{DarkGreen}{rgb}{0.1,0.5,0.1}
\definecolor{DarkRed}{rgb}{0.5,0.1,0.1}
\definecolor{DarkBlue}{rgb}{0.1,0.1,0.5}
\definecolor{Gray}{rgb}{0.2,0.2,0.2}

\usepackage[small]{caption}
\usepackage[pdftex]{hyperref}
\hypersetup{
    unicode=false,          
    pdftoolbar=true,        
    pdfmenubar=true,        
    pdffitwindow=false,      
    pdfnewwindow=true,      
    colorlinks=true,       
    linkcolor=DarkBlue,          
    citecolor=DarkGreen,        
    filecolor=DarkRed,      
    urlcolor=DarkBlue,          
    pdftitle={},
    pdfauthor={},
}

\usepackage[skip=2mm,indent=5mm]{parskip}

\usepackage{thmtools,thm-restate}
\newtheorem{theorem}{Theorem}
\newtheorem{corollary}[theorem]{Corollary}
\newtheorem{lemma}[theorem]{Lemma}
\newtheorem{proposition}[theorem]{Proposition}
\theoremstyle{definition}
\newtheorem{definition}{Definition}

\newtheorem{assumption}{Assumption}
\newtheorem*{remark}{Remark}
\newtheorem{example}{Example}

\newtheorem*{property*}{Property}

\newcommand{\thetaPO}{\theta_{\mathrm{PO}}}
\newcommand{\thetaSL}{\theta_{\mathrm{SL}}}

\newcommand{\cF}{\mathcal{F}}

\newcommand{\cW}{\mathcal{W}}

\newcommand{\cX}{\mathcal{X}}
\newcommand{\cU}{\mathcal{U}}

\newcommand{\PR}{\mathrm{PR}}
\newcommand{\DPR}{\mathrm{R}}
\newcommand{\R}{\mathrm{R}}

\newcommand{\distance}{\mathrm{dist}}
\newcommand{\PP}{\mathrm{P}}

\usepackage{tikz}

\usetikzlibrary{shapes,decorations,arrows,calc,arrows.meta,fit,positioning}
\tikzset{
    -Latex,auto,node distance =1 cm and 1 cm,semithick,
    state/.style ={circle, draw, minimum width = 0.8 cm},
    point/.style = {circle, draw, inner sep=0.04cm,fill,node contents={}},
    bidirected/.style={Latex-Latex,dashed},
    el/.style = {inner sep=2pt, align=left, sloped}
}

\newcommand{\cD}{\mathcal D}
\newcommand{\DBase}{{\cD_{\mathrm{orig}}}}

\DeclareMathOperator*{\argsup}{arg\,sup}
\DeclareMathOperator*{\argmax}{arg\,max}
\DeclareMathOperator*{\E}{\mathbb{E}}
\let\Pr\relax
\DeclareMathOperator*{\Pr}{\mathbb{P}}

\renewenvironment{abstract}
 {\small
  \begin{center}
  \bfseries \abstractname\vspace{-.5em}\vspace{0pt}
  \end{center}
  \list{}{%
    \setlength{\leftmargin}{20mm}
    \setlength{\rightmargin}{\leftmargin}%
  }
  \item\relax}
 {\endlist}

\title{Performative Power\footnote{Authors in alphabetical order.}\\ 
}

\makeatletter
\newcommand{\printfnsymbol}[1]{%
  \textsuperscript{\@fnsymbol{#1}}%
}
\makeatother

\usepackage{authblk}
\stepcounter{footnote}
\addtocounter{footnote}{-1}
\author[1]{Moritz Hardt}
\author[2]{Meena Jagadeesan}
\author[1]{Celestine Mendler-Dünner}
\affil[1]{Max-Planck Institute for Intelligent Systems, Tübingen}
\affil[2]{University of California, Berkeley}
\date{}                     
\begin{document}

\maketitle

\begin{abstract}
We introduce the notion of performative power, which measures the ability of a firm operating an algorithmic system, such as a digital content recommendation platform, to cause 
change in a population of participants. We relate performative power to the economic study of competition in digital economies. Traditional economic concepts struggle with identifying anti-competitive patterns in digital platforms not least due to the complexity of market definition. In contrast, performative power is a causal notion that is identifiable with minimal knowledge of the market, its internals, participants, products, or prices.
Low performative power implies that a firm can do no better than to optimize their objective on current data. In contrast, firms of high performative power stand to benefit from steering the population towards more profitable behavior. We confirm in a simple theoretical model that monopolies maximize performative power. A firm's ability to personalize increases performative power, while competition and outside options decrease performative power. On the empirical side, we propose an observational causal design to identify performative power from discontinuities in how digital platforms display content. This allows to repurpose causal effects from various studies about digital platforms as lower bounds on performative power. Finally, we speculate about the role that performative power might play in competition policy and 
antitrust enforcement
in digital marketplaces.
\end{abstract}

\section{Introduction}

Digital platforms pose a well-recognized challenge for antitrust enforcement. Traditional market definitions, along with associated notions of competition and market power, map poorly onto digital platforms. A core challenge is the difficulty of precisely modeling the interactions between the market participants, products, and prices.
An authoritative report, published by the \cite{stigler19}, details the many challenges associated with digital platforms, among them: 
\emph{“Pinpointing the locus of competition can also be challenging because the markets are multisided and often ones with which economists and lawyers have little experience. This complexity can make market definition another hurdle to effective enforcement.”}
Published the same year, a comprehensive report from the European Commission \citep{cremer2019competition} calls for 
\emph{“less emphasis on analysis of market definition, and more emphasis on theories of harm and identification of anti-competitive strategies.” }

Our work responds to this call by developing a normative and technical proposal for reasoning about power in digital economies, while relaxing the reliance on market definition. 
Our running example is a digital content recommendation platform. The platform connects content creators with viewers, while monetizing views through digital advertisement. Key to the business strategy of a firm operating a digital content recommendation platform is its ability to predict revenue for content that it recommends or ranks highly. 
Often framed as a supervised learning task, the firm trains a statistical model on observed data to predict some proxy of revenue, such as clicks, views, or engagement. Better predictions enable the firm to more accurately identify content of interest and thus increase profit.  

A second way of increasing profit is more subtle. The platform can use its predictions to \emph{steer} participants towards modes of consumption and production that are easier to predict and monetize. For example, the platform could reward consistency in the videos created by content creators, so that the audience and the popularity of their videos becomes more predictable. 
Similarly, the platform could recommend addictive content to viewers, appealing to behavioral weaknesses in order to drive up viewer engagement. How potent such a strategy is depends on the extent to which the firm is able to steer participants, which we argue reveals a salient power relationship between the platform and its participants.

\subsection{Our contribution}

We introduce the notion of \emph{performative power} that quantifies a firm’s ability to steer a population of participants. We argue that the sensitivity of participant behavior to algorithmic changes in the platform provides an important indicator of the firm's power.
Performative power is a causal statistical notion that directly quantifies how much participants change in response to actions by the platform, such as updating a predictive model. In doing so it avoids market specifics, such as the number of firms involved, products, and monetary prices. Neither does it require a competitive equilibrium notion as a reference point. Instead, it focuses on where rubber meets the road: the algorithmic actions of the platform and their causal powers. 

We first investigate the role of performative power in optimization. In particular, we build on recent developments in performative prediction~\citep{PZMH20} to articulate the fundamental difference between learning and steering in prediction. We show that under low performative power, a firm cannot do better than standard supervised learning on observed data. Intuitively, this means the firm optimizes its loss function \emph{ex-ante} on data it observes without the ability to steer towards data it would prefer. We interpret this optimization strategy as analogous to the firm being a price-taker, an economic condition that arises under perfect competition in classical market models. 
We contrast this optimization strategy with a firm that performs \emph{ex-post} optimization and benefits from steering towards data it prefers. Formally, we provide an upper-bound on the distance between the two solution concepts in terms of performative power.

Then, to study the qualitative properties of performative power we consider the concrete algorithmic market model of strategic classification. Strategic classification models participants as best-responding agents that change their features rationally in response to a predictor with the goal of achieving a better prediction outcome. 
In this simple setting, we show that the willingness of participants to invest in changing their features governs the performative power of the firm. We investigate the role of different economic factors by extending the standard model to incorporate competing firms and outside options. We highlight two key observations:
\begin{itemize}
\item A \emph{monopoly} firm can have significant performative power. In this case, performative power is derived because participants are willing to incur a cost up to the utility of using the service in order to adjust to the firm’s predictor. Moreover, performative power is maximized if a monopoly firm has the ability to personalize decisions to individual users.  
\item Performative power decreases in the presence of \emph{competition} and \emph{outside options}. In particular, when firms compete for participants, offering services that are perfect substitutes for each other, then even two firms can lead to zero performative power. This result stands in analogy with the classical Bertrand competition. 
\end{itemize}

On the empirical side, we propose a causal design to identify performative power in the context of a recommender system arranging content into display slots. This design, we call \textit{discrete display design (DDD)}, establishes a connection between performative power and the causal effect of display position on consumption. To derive a lower bound on performative power, DDD constructs a hypothetical algorithmic action that aggregates the causal effects of display position across the population. This allows us to repurpose reported causal effects of display position as lower bounds on performative power. It also charts out a concrete empirical strategy for understanding power in digital economies, both experimentally and observationally.

Finally, we examine the potential role of performative power in competition policy. We contrast performative power with traditional measures of market power, describe how performative power can capture complex behavioral patterns, and discuss the role that performative power might play in ongoing antitrust debates. 

\subsection{Related work}

Our notion of performative power builds on the development of performativity in prediction by \cite{PZMH20}. 
Performativity captures that the predictor can influence the data-generating process, a dependency ruled out by the traditional theory of supervised learning. A growing line of work on performative prediction, e.g., \citep{mendler20stochasticPP, drusvyatskiy2022stochastic, izzo2021learn, dong2021approximate, miller2021outside, brown2022performative, li2021state, raydecision,JZM22, wood22}, has studied different optimization challenges and solution concepts in performative prediction. Rather than viewing performative effects as an additional challenge for the learning algorithm, we argue that performativity reveals a salient power relationship between the decision maker and the population. From an optimization perspective, our work demonstrates that sufficiently high performative power is necessary for performative optimization approaches to achieve lower risk compared with standard supervised learning.

The \textit{strategic classification} setup we use for our case study was proposed in~\citep{bruckner12pred, hardt16strat} and is closely related to a line of work in the economics community~\citep{frankel2022improving, Ball2020, HG20, FK19}. A long line of work on strategic classification makes the assumption that performative effects are the result of individuals manipulating their features so as to best respond to the deployment of a predictive model. The focus has been on describing participant behavior in response to a single firm acting in isolation. Our extensions incorporate additional market factors into the model, such as outside options or the choice between competing firms, which we believe are helpful for gaining a better understanding of strategic interactions in digital economies. Beyond the case of a single classifier, recently, \cite{narang2022multiplayer} and \citet{piliouras22} analyzed settings with multiple firms that simultaneously apply retraining algorithms in performative environments. Similar to our analysis in Section~\ref{sec:optimization}, these works study the solution concept of a Nash equilibrium, however, with a focus on proving convergence to equilibrium solutions, whereas we are interested in how these equilibria interact with performative power. \citet{GZKZ21} study another model of feedback loops arising from competition between machine learning models.

There is extensive literature on the topic of competition on digital platforms that we do not attempt to survey here. For starting points, see, for example, recent work by~\citet{bergemann2022data}, a survey by \citet{calvano2021market}, a discussion by~\citet{parker2020digital}, the reports already mentioned~\citep{stigler19, cremer2019competition}, as well as a macroeconomic perspective on the topic~\citep{syverson19mp}.

\section{Performative power}

Fix a set $\cU$ of participants interacting with a designated firm, where each $u\in\cU$ is associated with a data point~$z(u)$. Fix a metric $\distance(z,z')$ over the space of data points. Let $\cF$ denote the set of actions a firm can take. We think of an action~$f\in\cF$ as a predictor that the firm can deploy at a fixed point in time. For each participant~$u\in\cU$ and action~$f \in\cF$, we denote by $z_f(u)$ the potential outcome random variable representing the counterfactual data of participant~$u$ if the firm were to take action~$f$.
\begin{definition}[Performative Power]
\label{def:performativepower}
Given a population $\cU$, an action set $\cF$, potential outcome pairs $(z(u),z_f(u))$ for each unit~$u\in\cU$ and action~$f\in\cF$, and a metric $\distance$ over the space of data points, we define the 
\emph{performative power} of the firm as
\[
\PP:=\sup_{f \in \cF}\; \frac 1 {|\cU|} \sum_{u \in \cU} \E\left[\distance\left(z(u),z_f(u)\right)\right]\,,
\]
where the expectation is over the randomness in the potential outcomes.
\end{definition}
The expression inside the supremum generalizes an average treatment effect, corresponding to scalar valued potential outcomes and the absolute value as metric. We could generalize other causal quantities such as heterogeneous treatment effects, but this avenue is not subject of our paper. The definition takes a supremum over possible actions a firm can take at a specific point in time. We can therefore lower bound performative power by estimating the causal effect of any given action $f\in\cF.$ 

Having specified the sets~$\cF$ and~$\cU$, estimating performative power amounts to causal inference involving the potential outcome variables~$z_f(u)$ for unit $u\in\cU$ and action $f\in\cF$. In an observational design, an investigator is able to identify performative power without an experimental intervention on the platform. We propose and apply one such observational design in Section~\ref{sec:DDD}. In an experimental design, the investigator deploys a suitably chosen action to estimate the effect. Neither route requires understanding the specifics of the market in which the firm operates. It is not even necessary to know the firm's objective function, how it optimizes its objective, and whether it successfully achieves its objective. In practice, the dynamic process that generates the potential outcome~$z_f(u)$ may be highly complex, but this complexity does not enter the definition. Consequently, the definition applies to complex multisided digital economies that defy mathematical specification. 
To make this abstract concept of performative power more concrete, we instantiate it in a concrete example.

\subsection{Running example: Digital content recommendation}

Consider a digital content recommendation platform, such as the video sharing services YouTube or Twitch. The platform aims to recommend channels that generate high revenue, personalized to each viewer. Towards this goal, the platform collects data to build a predictor~$f$ for the value of a channel~$c$ to a viewer with preferences~$p$. Let $x = (x_c, x_p)$ be the features used for the prediction task that capture attributes $x_c$ of the channel and the attributes $x_p$ of the  viewer preferences. Let~$y$ be the target variable, such as \emph{watch time}, that acts as a proxy for the monetary value of showing a channel to a specific viewer. For concreteness, take the supervised learning loss $\ell(f(x), y)$ incurred by a predictor~$f$ to be the squared loss $(f(x) - y)^2$. 

When defining performative power, participants could either be viewers or content creators. The definition is flexible and applies to both. 
By selecting the units $\cU$, which features to include in the data point $z$, and how to specify the distance metric $\distance$, we can pinpoint the power relationship we would like to investigate.

\paragraph{Content creators.} The predictor~$f$ can affect the type of videos that content creators stream on their channels. For example, content creators might strategically adjust various features of their content relevant for the predicted outcome, such as the length, type or description of their videos, to improve their ranking. Thus, by changing how it predicts the monetary value of a channel, the platform can induce changes in the content on the channel. To measure this source of power, we let the participants $\cU$ be content creators and suppose that each content creator $u\in\cU$ maintains a channel of videos. Let the data point $z(u)$ correspond to features $x_c$ characterizing the channel $c$ created by content creator $u$. Let $\distance$ be a metric over features of content. The resulting instantiation of performative power measures the changes in content induced by potential implementations $\cF$ of the prediction function and thus captures a power relationship between the platform and the content creators. In Section \ref{sec:strategic}, we investigate this form of performative power from a theoretical perspective by building on the setup of \textit{strategic classification}.

\paragraph{Viewers.} The predictor $f$ can shape the consumption patterns of viewers. In particular, viewers tend to follow recommendations when deciding what content to consume (e.g. \citep{ursu2018}). Thus, by changing which content it recommends to a user, the platform can induce changes in the target variable: how much time the user spends watching content on a given channel. Let's suppose that we wish to investigate the effect of the predictor on viewer consumption of a certain genre of content (e.g. radical content). To formalize this source of power, we let the units $\cU$ be viewers. Let the data point $z(u)$ correspond to how long the viewer $u$ spends watching content in the genre of interest. 
Let $\distance(z, z') = |z - z'|$ capture the difference in watch time. The resulting instantiation of performative power measures the changes in consumption of a given genre of content induced by a set of prediction functions $\cF$ the firm could implement. In Section~\ref{sec:DDD}, we propose an observational design to identify this quantity by establishing a formal connection to the causal effect of display position.

\section{Learning versus steering}
\label{sec:optimization}

Performative power enters the firm's optimization problem and has direct consequences for how a firm can achieve low risk. Instead of identifying the best action~$f$ while treating data as fixed, high performative power enables the firm to \emph{steer} the population towards data that it prefers. 
In the following, we elucidate the role of performative power in the optimization strategy of a firm and the equilibria attained in an economy of predictors.

\subsection{Optimization strategies}

We focus on predictive accuracy as the optimization objective of the firm. Hence, the goal of the firm is to choose a predictive model $f$ that suffers small loss $\ell(f(x),y)$ measured over instances $(x,y)$. To elucidate the role of steering we distinguish between the \emph{ex-ante} loss $\ell(f(x(u)),y(u))$ and the \emph{ex-post} loss $\ell(f(x_f(u)),y_f(u)).$ The former describes the loss that the firm can optimize when building the predictor. The latter describes the loss that the firm observes after deploying $f$. More formally, the  \textit{ex-post risk} that the firm suffers after deploying $f$ on a population $\cU$ is given by
\begin{equation}
\label{eq:expost}
 \frac{1}{|\cU|} \sum_{u \in \cU}\ell(f(x_f(u)), y_f(u))\,.
\end{equation}
Expression~\eqref{eq:expost} is an instance of what \cite{PZMH20} call \emph{performative risk} of a predictor. That is the loss a predictor incurs on the distribution over instances it induces. To simplify notation we adopt their conceptual device of a distribution map: let $\cD(\theta)$ map a predictive model, characterized by model parameters $\theta$, to a distribution over data instances. 

To express our setting within the framework of performative prediction, we assume the predictive model $f$ is parameterized by a parameter vector  $\theta\in\Theta$. We let a data instance correspond to $z(u)=(x(u),y(u))$ for $u\in\cU$ so we can capture performativity in the features as well as in the labels. Then, the \textit{aggregate distribution} over data $\cD(\theta)$  corresponds to the distribution over the potential outcome variable $z_{\theta}(u)$ after the firm takes action $\theta$, where the randomness comes from $u$ being uniformly drawn from $\cU$ as well as randomness in the potential outcomes. The firm's ex-post risk~\eqref{eq:expost} from deploying predictor $f_{\theta}$ corresponds to the performative risk:
\[
\PR(\theta) := \E_{z \sim \cD(\theta)}[\ell(\theta;\,z)] \]
where the loss typically corresponds to the mismatch between the predicted label and the true label: $\ell(\theta;\,z)=\ell(f_\theta(x), y)$ for $z=(x,y).$

In performative risk minimization, observe that $\theta$ arises in two places in the objective: in the distribution $\cD({\theta})$ and in the loss $\ell(\theta;\,z)$. Thus, for any choice of model $\phi$, we can decompose the performative risk $\PR(\theta)$ as: 
\begin{equation}
\label{eq:decomposition}
\PR(\theta)=\DPR(\phi,\theta) + \left(\DPR(\theta,\theta)- \DPR(\phi,\theta)\right)
\end{equation}
where
$\DPR(\phi,\theta):=\E_{z\sim\cD(\phi)}\ell(\theta;\,z)\,$ denotes the loss of a model $\theta$ on the distribution $\cD(\phi)$. 
This tautology highlights the difference between learning and steering and we differentiate between the following two optimization approaches:

\paragraph{Ex-ante optimization.} Ex-ante optimization  focuses on optimizing the first term in the decomposition~\eqref{eq:decomposition}. For any $\phi$, the resulting minimizer can be computed statistically:  
\[\thetaSL:=\arg\min_{\theta\in\Theta}\, \DPR(\phi,\theta).\] 
Let $f_\phi$ be any previously chosen model, then employing supervised learning on historical data sampled from $\cD(\phi)$ corresponds to what we call ex-ante optimization. 

\paragraph{Ex-post optimization.} In contrast to ex-ante optimization, \textit{ex-post optimization} accounts for the impact of the model on the distribution. It trades-off the two terms in \eqref{eq:decomposition}, and directly optimizes the performative risk \[\thetaPO:=\arg\min_{\theta\in\Theta} \,\PR(\theta).\]
Solving this problem exactly, and finding the performative optimum $\thetaPO$ requires optimization over the distribution map $\cD(\theta)$. 

In the context of digital content recommendation, ex-ante optimization corresponds to training the model  $\theta$ on historical data collected by the platform,  whereas ex-post optimization selects $\theta$ based on randomized experiments, A/B testing or explicit modeling of $\cD(\theta)$.
It holds that $\PR(\thetaPO)\leq \PR(\thetaSL)$, because in ex-post optimization the firm can choose to steer the population towards more predictable behavior. High ex-post predictability may be an objective worth pursuing for firms relying on predictive optimization~\citep{shmueli20}, as speculated on in popular science writing~\citep{ward22}. 

\begin{remark}[Generalizing to other objectives]
Note that we focus on predictive accuracy as an objective function. Nonetheless, the conceptual distinction between learning and steering applies to \textit{general optimization objectives}. Ex-ante optimization corresponds to optimizing on historical data, whereas ex-post optimization corresponds to implicitly or explicitly optimizing over the counterfactuals. 
\end{remark}

\subsection{Gain of ex-post optimization is bounded by a firm's performative power}

We show that the gain of ex-post optimization over ex-ante optimization can be bounded by the firm's performative power with respect to the set of actions $\Theta$ and the data vector $z=(x,y).$ Intuitively, if the firm's performative power is low, then the distributions $\cD(\theta)$ and $\cD(\phi)$  for any $\theta,\phi\in\Theta$ are close to one another. This distributional closeness, coupled with a regularity assumption on the loss, means that the second term in \eqref{eq:decomposition} should be small. Thus, using the ex-ante approach of minimizing the first term produces a near-optimal ex-post solution, as we demonstrate in the following result: 
\begin{proposition}
\label{prop:SL}
Let $\PP$ be the performative power of a firm with respect to the action set $\Theta$. Let $L_z$ be the Lipschitzness of the loss in $z$ with respect to the metric $\distance$. Let $\thetaPO$ be the ex-post solution and $\thetaSL$ be the ex-ante solution computed from $\cD(\phi)$ for any past deployment $\phi \in \Theta$. Then, we have that:
\[ \PR(\thetaSL) \le \PR(\thetaPO) + 4L_z \PP.   \]
If $\ell$ is $\gamma$-strongly convex, we can further bound the distance between $\thetaSL$ and $\thetaPO$ in parameter space as:
\[\|\thetaSL- \thetaPO\|_2 \le \sqrt{\frac{8 L_z \PP}{\gamma}}.\]
\end{proposition}

Proposition~\ref{prop:SL} illustrates that the gain achievable through ex-post optimization is bounded by performative power. Thus, a firm with small performative power cannot do much better than ex-ante optimization and might be better off sticking to classical supervised learning practices instead of engaging with ex-post optimization. 

\subsection{Ex-post optimization in an economy of predictors}
\label{sec:expostmarket}

The result in Proposition \ref{prop:SL} studies the optimization strategy of a single firm in isolation. In this section, we investigate the interaction between the strategies of multiple firms that optimize simultaneously over the same population. We consider an idealized marketplace where $C$ firms all engage in ex-post optimization and we assume all exogenous factors remain constant. Let $\cD(\theta^1, \ldots, \theta^{i-1}, \theta^i, \theta^{i+1}, \ldots, \theta^C)$ be the distribution over $z(u)$ induced by each firm $i\in[C]$ deploying model $f_{\theta^i}$. Let $\ell_i$ denote the loss function chosen by firm $i$. We say a set of predictors $[f_{\theta^1}, \ldots, f_{\theta^C}]$ is a \emph{Nash equilibrium} if and only if no firm has an incentive to unilaterally deviate from their predictor using ex-post optimization: 
\[\theta^i \in \argmin_{\theta\in\Theta}\E_{z \sim \cD(\theta^1, \ldots, \theta^{i-1}, \theta, \theta^{i+1}, \ldots, \theta^C)}[\ell_i(\theta;\,z)].\]
First, we show that at the Nash equilibrium, the suboptimality of each predictor $f_{\theta^i}$ on the induced distribution depends on the performative power of the respective firm. 
\begin{proposition}
\label{prop:equilibrium}
Suppose that the economy is in a Nash equilibrium $(\theta^1, \ldots, \theta^C)$, and firm $i$ has performative power $\PP_i$ with respect to the action set $\Theta$. Let $L_z$ be the Lipschitzness of the loss $\ell_i$ in $z$ with respect to the metric $\distance$. Then, it holds that: 
\[\E_{z \sim \mathcal{D}} [\ell_i(\theta^i;\,z)] \le \min_{\theta} \;\E_{z \sim \mathcal{D}} [\ell_i(\theta;\,z)] + L_z \PP_i\,,  \]
where $\mathcal{D} = \mathcal{D}(\theta^1, \ldots, \theta^C)$ is the distribution induced at the equilibrium. If $\ell_i$ is $\gamma$-strongly convex, then we can also bound the distance between $\theta^i$ and $\arg\min_{\theta\in\Theta} \E_{z \sim \mathcal{D}} [\ell_i(\theta;\,z)]$ in parameter space. 
\end{proposition}

Proposition \ref{prop:equilibrium} implies that if the performative power of all firms is small ($\PP_i\rightarrow 0\;\forall i$),  then the equilibrium becomes indistinguishable from that of a static, non-performative economy with distribution $\cD$ over content. However, there is an interesting distinction between such a Nash equilibrium and the static setting: if the firms were to pursue a different strategy and decided to collude---for example, because of common ownership~\citep{azar18}\footnote{Common ownership refers to the situation  where many competitors are jointly held by a small set of large institutional investors~\citep{azar18}. }
---then they would be able to significantly shift the distribution. 

\paragraph{Mixture economy.}
Next, we analyze the behavior of multiple firms optimizing simultaneously. We consider a \textit{mixture economy}, where all of the firms share a common loss function $\ell$ and performative power is uniformly distributed across firms. Let $z(u),z^{C=1}_{\theta}(u)$ denote the pair of counterfactual outcomes before and after the deployment of $\theta$ in a hypothetical  monopoly economy where a single firm holds all the performative power. Let $\cD^{C=1}(\theta)$ be the distribution map associated with the variables $z^{C=1}_{\theta}(u)$ for $u\in \cU$. In a uniform mixture economy, we assume that each participant $u\in\cU$ uniformly chooses one of the $C$ firms. Consequently, the counterfactual $z_{\theta}(u)$ associated with one firm changing its predictor to $\theta$ is equal to $z(u)$ with probability $1-1/C$ and $z^{C=1}_{\theta}(u)$ otherwise. We can apply Proposition \ref{prop:equilibrium} to analyze the equilibria in the limit as $C \rightarrow \infty$.
\begin{corollary}
\label{cor:eq}
Suppose that all firms $i\in[C]$ share the same loss function $\ell_i=\ell$. Let $\theta^*$ be a symmetric Nash equilibrium in the mixture economy with $C$ platforms. As $C \rightarrow \infty$, it holds that:
\[\E_{z \sim \cD(\theta^*, \ldots, \theta^*)} [\ell(\theta^*;\,z)] \;\rightarrow\;  \min_{\theta} \E_{z \sim \cD^{C=1}(\theta^*)} [\ell(\theta;\,z)].\]
\end{corollary}
Corollary \ref{cor:eq} demonstrates that a symmetric equilibrium approaches a \textit{performatively stable point} of $\cD^{C=1}$ as the number of firms in the economy grows and the performative power of each individual firm diminishes. In contrast, if  all $C$ firms collude, their performative power adds up and they would obtain the performative power of a monopoly platform. As a consequence, the firms would take advantage of their collective power and all choose a \textit{performatively optimal point} of $\cD^{C=1}$---recovering the equilibrium in a monopoly economy with a single firm. 
Since performatively optimal and performatively stable points can be arbitrarily far apart in general~\citep{miller2021outside}, a competitive economy of optimizing firms can exhibit a significantly different equilibrium from that of the monopoly or collusive economy.

\section{Performative power in strategic classification}
\label{sec:strategic}

We now turn to a stylized market model and investigate how performative power depends on the economy in which the firm operates. Specifically, we use \textit{strategic classification}~\citep{hardt16strat} as a test case for our definition. In strategic classification, participants strategically adapt their features with the goal of achieving a favorable classification outcome. Hence, performative power is determined by the degree to which a firm's classifier can impact participant features. We use this concrete market setting to examine the qualitative behavior of performative power in the presence of competition and outside options.

\subsection{Strategic classification setup}

Let $x(u)$ be the \textit{features} and $y(u)$ the \textit{binary label} describing a participant $u\in\cU$. A firm chooses a binary predictor $f: \mathbb{R}^m \rightarrow \left\{0,1\right\}$ and incurs loss $\ell(f(x),y) = |f(x) - y|$. 
Let  $\cD_{\text{orig}}$ denote the base distribution over features and labels $(x_{\text{orig}}(u),y_{\text{orig}}(u))$ absent any strategic adaptation, which we assume is continuous and supported everywhere. Let $\cD(f)$ be the distribution over potential outcomes $(x_f(u),y_f(u))$ that arises from the response of participant $u$ to the deployment of a model $f$. We assume that participant $u$ incurs a cost $c(x_{\text{orig}}(u),x')$ for changing their features to $x'$. In line with the standard strategic classification setup, the cost for feature changes is measured relative to the \textit{original} features. We further assume that~$c$ is a metric, in particular, any feature change that deviates from the original features results in nonnegative cost for participants.  
Further, we assume the label does not change, i.e., $y_f(u)=y_{\text{orig}}(u)$.

\paragraph{Instantiation of performative power.} We measure performative power over the data vector $z(u)=x(u)$, reflecting that strategic behavior impacts the feature vector that enters the prediction function. 
Then, the choice of distance metric enables us to define how to weight specific feature changes. 
For instance, in our running example of digital content recommendations where participants correspond to content creators, performative power measures how much the content of each channel changes with changes in the recommendation algorithm. 
If we are interested in the burden on \emph{content creators}, we choose the distance metric to be aligned with the cost function $c$ of producing a piece of content. However, if we are interested in measuring the impact of changes in content on viewers, a distance metric that reflects harm to viewers might be more appropriate. We  keep this distance metric abstract in our analysis. 

\begin{figure}[t!]
\subfigure{\label{fig:shift}
\includegraphics[width=0.5\textwidth]{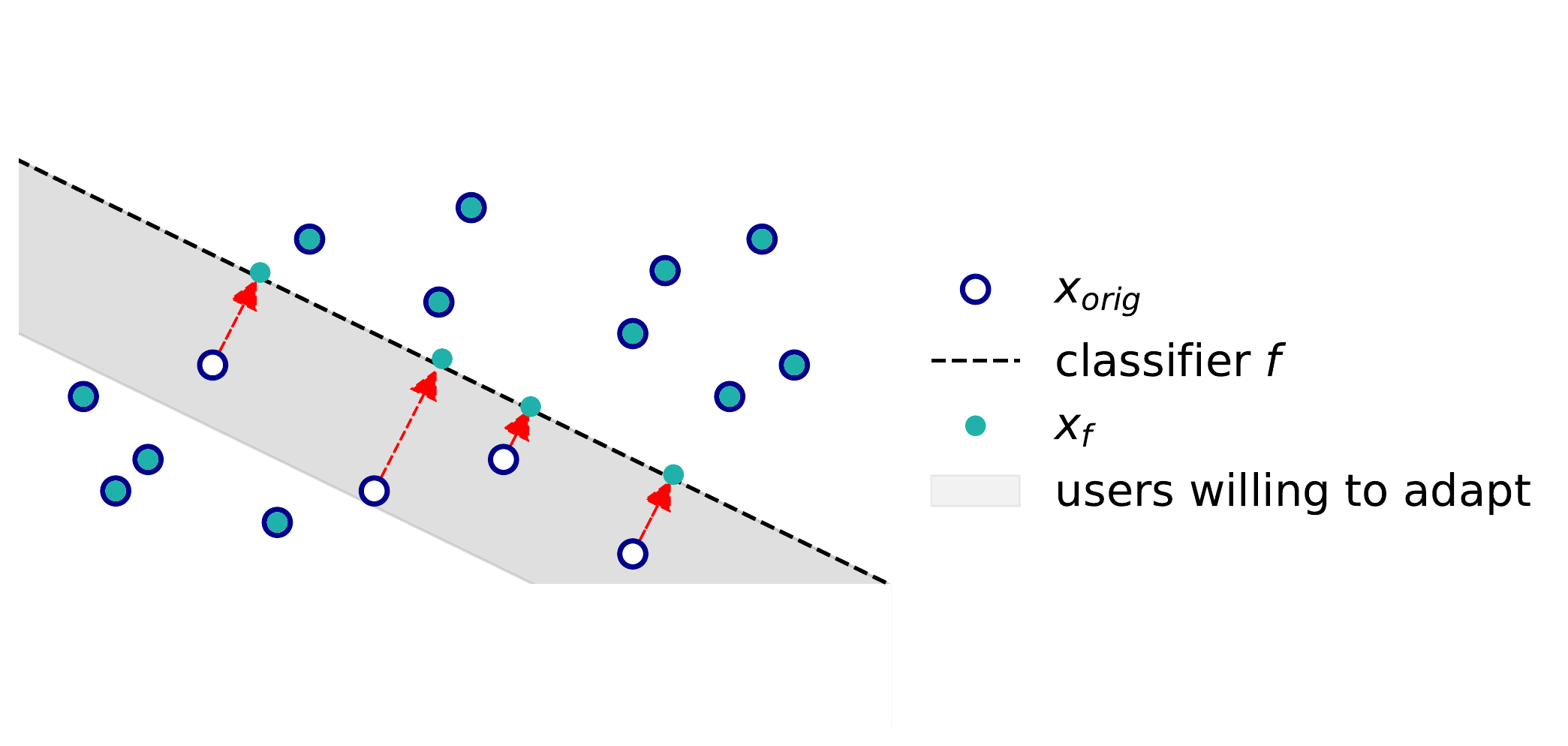}
}
\subfigure{\label{fig:B}
\centering
\begin{tikzpicture}[line/.style={>=latex}] 
\coordinate (V1) at (1.5, -0.7);
\coordinate (V0) at (0,0);
\coordinate (V3) at (0.6,1.6);
\filldraw[color=gray, fill=gray,fill opacity=0.2] (V0) circle (1.8);
\draw[color=white] (-3,-0.5) circle (1pt);
\draw[color=white] (4,-0.5) circle (1pt);
\draw[->, line, color=black, thick] (V0) -- node [below, near end]   {${x_{f}}$} (V1);
\draw[<->, line, color=red, thick, densely dotted] (V1) -- node [above, right=1pt] {$\distance(x_f,x_f')$} (V3) ;
\draw[->, line, color=black, thick] (V0) -- node [left] {$x_{f'}$} (V3);
\filldraw[color=DarkBlue,fill=white,thick] (V0) circle (2.5pt) node[below]{$x_{\text{orig}}$};
\draw[color=black](0.4,-1.8) circle (0pt) node[below]{\footnotesize{$\mathcal X_{\Delta \gamma}(u):=\{x:c(x_{\text{orig}}(u), x) \leq \Delta \gamma\}$}};
\end{tikzpicture}}
\caption{Illustrations for 2-dimensional strategic classification example. (left) Participants behave differently depending on their relative position to the decision boundary. (right) Visualization of participant expenditure constraint $\mathcal X_{\Delta \gamma}(u)$.}
\label{fig:SC}
\end{figure}
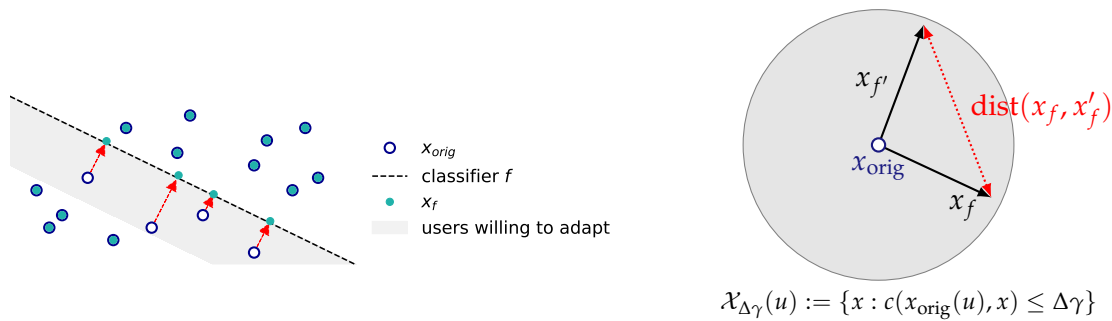

\subsection{Performative power in the monopoly setting}
\label{sec:sc-monopoly}

Consider a single firm that offers utility~$\gamma>0$ to its participants for a positive classification. We assume that participants want to use the service regardless of what classifier the firm chooses. In addition, assume there is an \emph{outside options} at utility level $\beta > 0$. This decreases the budget participants are willing to invest to their \emph{surplus utility} $\Delta\gamma=\max(0,\gamma-\beta)$. 
We adopt the following standard rationality assumption on participant behavior. 
\begin{assumption}[Participant Behavior Specification] 
\label{ass:behavior}
Let $\Delta \gamma \ge 0$ be the surplus utility that a participant can expect from a positive classification outcome from classifier $f$ over any outside option. Then, a participant $u \in \cU$ with original features $x_{\text{orig}}(u)$ will change their features according to 
\[x_f(u)=\argmax_{x'} \left( \Delta \gamma f(x') - c(x_{\text{orig}}(u),x')\right)\,.\]
\end{assumption}

Assumption~\ref{ass:behavior} guarantees that a participant will change their features if and only if the cost of a feature change is no larger than $\Delta \gamma$. Furthermore, if participants change their features, then they will expend the minimal cost required to achieve a positive outcome. For $\beta=0$ this recovers the typical strategic classification setup proposed by~\citet{hardt16strat}.
This specification of participant behavior allows us to bound performative power in terms of the cost function $c$ and the distance function $\distance$. Namely, the potential values that $x_f(u)$ can take on is restricted to
\begin{equation}
    \mathcal X_{\Delta \gamma}(u):=\{x:c(x_{\text{orig}}(u), x) \leq \Delta \gamma\}\,.
\end{equation}
Thus, the effect of a change to the decision rule on an individual participant $u$ can be upper bounded by the distance between $x(u)$ and the most distant point in $\cX_{\Delta \gamma}$. Aggregating these unilateral effects yields a bound on performative power:

\begin{lemma}
\label{lemma:Pdiam}
The performative power $\PP$ of the firm with respect to any set of predictors $\mathcal{F}$ can be upper bounded as:
\begin{equation}\PP \le \frac 1 {|\cU|}\sum_{u\in\cU}\;\sup_{x'\in\mathcal X_{\Delta\gamma}(u)}\mathrm{dist}(x(u),x')
\label{eq:ubPP}
\end{equation}
\end{lemma}

 If the firm action space $\mathcal{F}$ is restricted to a parameterized family, the upper bound in Lemma \ref{lemma:Pdiam} need not be tight. In particular, a typical decision rule, such as a linear threshold classifier, does not impact all participants $u\in\cU$ equally (the amount of change that the firm can induce with a decision rule $f$ on an individual $u$ depends the relative position of their features $x_{\mathrm{orig}}(u)$ to the decision boundary, as we visualize in Figure~\ref{fig:SC}). Thus, the firm can't necessarily extract the full utility from all participants simultaneously. We quantify this gap for a 1-dimensional example in Appendix~\ref{app:ex-heter}. 

\paragraph{Personalization.} Interestingly, the ability to fully \textit{personalize} decisions to each user maximizes a firm's performative power. To capture this, let the first coordinate of the features $x(u)$ be the index of the user in the population and suppose that this coordinate is immutable. In this case, we can precisely pin down the performative power as follows:
\begin{proposition}
\label{prop:personalized}
Consider a population $\cU$ of users. Suppose that  the first coordinate is immutable: that is, $c(x, x') = \infty$ if $x_1 \neq x'_1$ and  $(x_{\mathrm{orig}}(u))_1 = (x(u))_1 = i$ where $i$ is the index of user $u$. Then, the performative power with respect to the set $\mathcal{F}$ of all functions from $\mathbb{R}^m$ to $\left\{0,1\right\}$ is given by:
\[\PP = \frac 1 {|\cU|}\sum_{u\in\cU}\;\sup_{x'\in\mathcal X_{\Delta\gamma}(u)}\mathrm{dist}(x(u),x').\]
\end{proposition}
Proposition \ref{prop:personalized} demonstrates that when firms can fully personalize their decisions to each user, the upper bound in Lemma \ref{lemma:Pdiam} is in fact tight. In particular, the firm is able to extract maximum utility from each user, despite the heterogeneity in the population.

\paragraph{Value of the service.} We investigate the role of $\Delta \gamma$ in the upper bound of Lemma~\ref{lemma:Pdiam}. Recall that user behavior is determined by the cost $c$ of changing features relative to $x_{\text{orig}}(u)$, performative power is measured as the distance from the current state $x(u)$ with respect to $\distance$ (see Figure~\ref{fig:SC}). The Lipschitz constant 
\[L := \sup_{x, x'} \frac{\distance(x, x')}{c(x, x')}\] 
allows us to relate the two metrics and derive a simpler bound:
\begin{corollary}
\label{prop:BSC}
The performative power $\PP$ of a firm in the monopoly setup with respect to any set of predictors $\mathcal{F}$ can be bounded as:  
\begin{equation}
\label{eq:boundSC}
\PP\leq 2 L \,\Delta \gamma.\end{equation}
where $\Delta \gamma$ measures the surplus utility offered by the service of the firm over outside options. 
\end{corollary}

Corollary \ref{prop:BSC} makes explicit that  $\Delta \gamma>0$ is a prerequisite for a firm to have any performative power, even in a monopoly economy. This qualitative behavior of performative power is in line with common intuition in economics that monopoly power relies on the firm offering a service that is superior to existing options.

\subsection{Firms competing for participants}\label{subsec:strategiccompete}

We next consider a model of \textit{competition} between two firms where participants always choose the firm that offers higher utility. In this model of perfectly elastic demand, we demonstrate how the presence of competition  reduces the performative power of a firm. In particular, we will show that for a natural constraint on the firm's action set, each firm's performative power can drop to zero at equilibrium, regardless of how much utility participants derive from the firm's service. 

To model competition in strategic classification, we specify participant behavior as follows: 
Given that the first firm deploys $f_1$ and the second firm deploys $f_2$, then participant $u$  will choose the first firm if 
$\max_{x'} \left(  f_1(x') - c(x_{\text{orig}}(u),x')\right) > \max_{x'} \left( f_2(x') - c(x_{\text{orig}}(u),x')\right)$, and choose $f_2$ otherwise. A participant tie-breaks in favor of the lower threshold, randomizing if they are equal. After choosing firm $i \in \left\{1,2\right\}$, they change their features according to Assumption \ref{ass:behavior} as $x_f(u) = \argmax_{x'} \left(\gamma f_i(x') - c(x_{\text{orig}}(u),x')\right)$, where $\gamma$ is the utility of a positive outcome.

We assume that the firm chooses their classifier based on the following utility function. For a rejected participant, the firm receives utility $0$ and for an accepted participant, the firm receives utility  $\alpha>0$ if they have a positive label and utility $-\alpha$ if they have a negative label. 
We assume that the firm's action set is constrained to models for which it derives non-negative utility. More specifically, if $f_{\theta}$ denotes the model deployed by the competing firm, let the action set $\cF^+(\theta)$ of this firm denote the set of models that yield non-negative utility for the firm.

For simplicity, focus on a 1-dimensional setup where $\cF$ is the set of threshold functions. We assume that  the cost function $c(x, x')$ is continuous in both of its arguments, strictly increasing in $x'$ for $x' > x$, strictly decreasing for $x' < x$, and satisfies $\lim_{x' \rightarrow \infty} c(x, x') = \infty$. Furthermore, we assume that the posterior $p(x) = \mathbb{P}_{\DBase}[Y = 1 \mid X = x]$ is strictly increasing in $x$ with $\lim_{x \rightarrow -\infty} p(x) = 0$, and $\lim_{x \rightarrow \infty} p(x) = 1$.

We show that the presence of competition  drives the performative power of each firm to zero. 

\begin{proposition}
\label{cor:perfpower}
Consider the 1-dimensional setup with two competing firms specified above. Suppose that the economy is at a symmetric Nash equilibrium $(\theta^*, \theta^*)$. If $L < \infty$, then the performative power of either firm with respect to the action set $\cF^+(\theta^*)$ is 
\[\PP = 0.\] 
\end{proposition}

The intuition behind Proposition~\ref{cor:perfpower} is that 
performative power of a firm purely arises from how much larger the current threshold $\theta$ is than the minimum threshold a firm can deploy within their action set $\cF^+(\theta)$. At the Nash equilibrium (where both firms best-respond with respect to their utility functions taking their own performative effects into account), the firms deploy exactly the minimum threshold within their action set. The formal proof of the result can be found in Appendix~\ref{app:competition}.

Proposition~\ref{cor:perfpower} bears an intriguing resemblance to well-known results on market power under Bertrand competition in economics (see e.g.,~\citep{Baye2008-bertrand}) that show how a state of zero power is reached in classical pricing economies with only two competing firms.

\section{Discrete display design}
\label{sec:DDD}

Now that we have examined the theoretical properties of performative power, we turn to the question of \textit{measuring} performative power from observational data. We focus on our running example of digital content recommendation and  propose an observational design to measure the recommender system's ability to shape consumption patterns through the arrangement of content. 


\subsection{The causal effect of position}
We assume that there are $C$ pieces of content $\mathcal{C} = \{0, 1, 2,\dots, C-1\}$ that the platform can present in $m$ display slots. We make the convention that item $0$ corresponds to leaving the display slot empty. We focus on the case of two display slots ($m=2$) since it already encapsulates the main idea. The first display slot is more desirable as it is more likely to catch the viewer's attention. 
Researchers have investigated the causal effect of position on consumption, often via quasi-experimental methods such as regression discontinuity designs, but also through experimentation in the form of A/B tests.

\begin{definition}[Causal effect of position]
Let the treatment $T\in\{0,1\}$ be the action of flipping the content in the first and second display slots for a viewer $u$, and let the potential outcome variable~$Y_t(u)$ indicate whether, under the treatment~$T=t$, viewer~$u$ consumes the content that is initially in the first display slot. We call the corresponding average treatment effect
\[\beta = \big|\frac{1}{|\cU|}\sum_{u\in\cU}\mathbb{E}\left[Y_1(u) - Y_0(u)\right]\big|\]
the \textit{causal effect of position} in a population of viewers $\cU$, where the expectation is taken over
the randomness in the potential outcomes. 
\label{def:display-effect}
\end{definition}

For example, \citet{NK15} estimate the causal effect of position in search advertising, where advertisements are displayed across a number of ordered slots whenever a keyword is searched. They measured the causal effect of position on click-through rate of participants. 

\subsection{From causal effect of position to performative power}
The identification strategy we propose, called \emph{discrete display design (DDD)}, derives a lower bound on performative power by repurposing existing measures of the causal effect of position.  Note that we focus on content recommendation in this section, the design however can be generalized to other settings where the firm's action corresponds to a discrete decision of how to display content. Setting up the DDD involves two steps: First, we need to instantiate the definition of performative power with a suitable action set which we choose such that one of the firm's actions result in swapping the position of content items, and second, we plug in the causal effect of position to lower bound performative power.

While the first step is mostly a technical exercise, the second step relies on a crucial assumption. In particular, it involves relating the unilateral causal effect of position to performative power that quantifies the effect of an action on the entire population of viewers. Thus, for being able to extrapolate the effect from a single viewer to the population DDD relies on a non-interference assumption. In the advertising example, this means that the ads shown to one viewer do not influence the consumption behavior of another viewer.  We investigate the two steps in  detail:

\paragraph{Step 1: Instantiating performative power.} 
Let the units~$\cU$ be the set of viewers. 
For each viewer $u\in\cU$ let $z(u)\in\mathbb{R}^C$ be the distribution over content items $\mathcal{C}$ consumed by viewer~$u$, represented as a histogram. More formally, let $z(u)$ be a vector in the $C$-dimensional probability simplex where the $i$th coordinate is the probability that viewer~$u$ consumes content item~$i$. The metric $\distance(z, z')$ is the $\ell_1$-distance $\distance(z, z') = \sum_{i=0}^{C-1} |z[i]-z'[i]|$. 

The decision space $\cF$ of the firm corresponds to its decisions of how to arrange content in the~$m=2$ display slots. It is natural to decompose this decision into a continuous score function~$s$ followed by a discrete conversion function~$\kappa$ that maps scores into an allocation. The score function $s\colon \cU \rightarrow \mathbb{R}^{C}$ maps the viewer to a vector of scores, where each coordinate is an estimate of the quality of the match between the viewer and the corresponding piece of content. The conversion function~$\kappa\colon \mathbb{R}^{C} \rightarrow \mathcal{C}^2$ takes as input the vector of scores and outputs an ordered list of items with the top~$2$ scores. We assume the platform displays these~$2$ items in order and the conversion function~$\kappa$ is fixed. Hence, we identify the firm's action space with the set of feasible score functions $\mathcal{S}\subseteq \cU\rightarrow\mathbb{R}^C$. 

To define the reference state $z(u)$, we think of~$s_\text{curr}$ as being  the score function currently deployed by the platform. Let~$\delta$ be the maximum difference in the highest score and second highest score for any user under $s_{\text{curr}}.$ 
Consider the set $\mathcal{S}$ of local perturbations to the scoring function $s_{\text{curr}}$ defined as \[\mathcal{S} := \left\{s\colon\cU \rightarrow \mathbb{R}^C \mid \forall u\in\cU\colon \|s(u) - s_\text{curr}(u)\|_{\infty} \le \delta \right\}.\] 
Notably, there exists an $s_\mathrm{swap}\in\mathcal S$ that is capable of swapping the order of the first and second highest scoring item under~$s_{\text{curr}}$ for any user $u\in\mathcal U$ simultaneously. We denote the counterfactual variable corresponding to a score function~$s\in\mathcal{S}$ as~$z_s(u)$. Given this specification, performative power with respect to the action set~$\mathcal{S}$ can be bounded by the causal effect of $s_\mathrm{swap}$ as follow
\begin{equation}\PP = \sup_{s \in \mathcal{S}}\; \frac{1}{\cU} \sum_{u \in \cU} \|z_{s_\text{curr}}(u) - z_s(u)\|_{1}\geq \frac{1}{\cU} \sum_{u \in \cU} \|z_{s_{\text{curr}}}(u) - z_{s_\mathrm{swap}}(u)\|_{1}
\label{eq:boundPP}.\end{equation}

\paragraph{Step 2: Lower bounding performative power.}
To relate the lower bound on performative power from \eqref{eq:boundPP} to the causal effect of position, let  $i_{\text{top}}(u)=\kappa\circ s_{\text{curr}}(u)[1]$ denote the coordinate of the item displayed to user $u$ in the first display slot under $s_\text{curr}$. Then, we can lower bound each term in the sum~\eqref{eq:boundPP} as \[\|z_{s_{\text{curr}}}(u) - z_{s_\mathrm{swap}}(u)\|_{1}\geq |z_{s_{\text{curr}}}(u)[i_\text{top}(u)] - z_{s_\mathrm{swap}}(u)[i_\text{top}(u)]|.\] 
Now, to enable us to study the effect of changing ${s_{\text{curr}}}$ to ${s_{\text{swap}}}$ independently for each user we place the following non-interference assumption on the counterfactual variables which  closely relates to the stable unit treatment value assumption (SUTVA)~\citep{imbens2015causal} prevalent in causal inference.
\begin{assumption}[No interference across units]
\label{assumption:independence}
For any $u \in \mathcal{U}$ and any pair of scoring functions $s_1, s_2 \in \mathcal{S}$, if $\kappa(s_1(u)) = \kappa(s_2(u))$, it also holds that $z_{s_1}(u) = z_{s_2}(u)$.
\end{assumption} 
 The assumption requires that there are no spill-over or peer effects and the content a viewer consumes only depends on the content recommended to them and not the content recommended to other viewers.
 The last step is to see that the effect of a unilateral change to the consumption of item $i_\text{top}(u)$ under $s_\text{swap}$ exactly corresponds to what we defined as the causal effect of position. Aggregating these unilateral causal effects across all viewers in the population we obtain a lower bound on performative power. The proof of Theorem~\ref{thm:swapping} can be found in Appendix~\ref{app:ddd}.
\begin{theorem}
\label{thm:swapping}
Let $\PP$ be performative power as instantiated above. If Assumption \ref{assumption:independence} holds, then performative power is at least as large as the causal effect of position 
\[\PP \ge \beta.\]
\end{theorem}

As a case study, consider the search advertisement marketplace of ~\cite{NK15}. We can leverage Theorem~\ref{thm:swapping} to relate the findings of their observational causal design to performative power. In particular, \citet{NK15} examine position effects in search advertising, where ads are displayed across a number of ordered slots whenever a keyword is searched. They found that the effect of showing an ad in display slot 1 versus display  slot 2 corresponds to $0.0048$ clicks per impression (see Table 2 in their paper). 
By treating each incoming keyword query as a distinct ``viewer'' $u$, this number exactly corresponds to what we defined as the causal effect of position. Thus, we can apply Theorem \ref{thm:swapping} to get $P \ge 0.0048$. Putting this into context; the mean click-through rate in display slot 2 is $0.023260$. Hence, the lower bound $0.0048$ is a $21\%$ percent increase relative to the baseline. The firm thus has a substantial ability to shape what advertisements users click on.

\section{Discussion}
\label{sec:discussion}
We discuss the potential role of performative power in competition policy and antitrust enforcement. The complexity of digital marketplaces has made it necessary to develop new approaches for evaluating and regulating these economies. One challenge is that traditional measures of market power---such as the Lerner Index~\citep{lerner34lernerindex}, or the Herfindahl–Hirschman Index (HHI)---are based on classical pricing markets for homogeneous goods, but these markets map poorly to digital economies. 
In particular, these measures struggle to appropriately capture the multi-sided nature of digital economies, to describe the multi-dimensionality of interactions, and to account for the role of behavioral weaknesses of consumers---such as tendencies for single-homing, vulnerability to addiction, and the impact of framing and nudging on participant behavior~\citep[e.g.][]{thaler08nudge, fogg02}. We further expand on this in Appendix~\ref{sec:discussion-app}.

By focusing on directly observable statistics, performative power could be particularly helpful in  markets that resist a clean mathematical specification. Performative power is sensitive to the market nuances without explicitly modeling them. For example, suppose that as a result of uncertainty about market boundaries, a regulator failed to account for a competitor in a marketplace. Performative power would still implicitly capture the impact of the competitor and indicate the market is more competitive than suspected.

We leave open the question of how to best instantiate performative power in a given marketplace. Conceptually, we view performative power as a tool to flag market situations that merit further investigation, since it corresponds to ``potential for harm to users''. However, if a regulator wishes to draw fine-grained conclusions about consumer harm, it is crucial to appropriately instantiate the choice of action set $\cF$, the definition of a population $\cU$, and the specification of the features $z$. As an example, we show in Appendix~\ref{appendix:optimizationSC} how to closely relate performative power into consumer harm for strategic classification. In general, however, harm and power are two distinct normative concepts, and going from performative power to consumer harm thus requires additional substantive arguments.

\section*{Acknowledgments}

We're grateful to Guy Aridor, Dirk Bergemann, Emilio Calvano, Gabriele Carovano, and Martin C.~Schmalz for helpful feedback and pointers.  MJ acknowledges support from the Paul and Daisy Soros Fellowship and the Open Phil AI Fellowship. 

\newpage
\bibliography{ref}

\begin{thebibliography}{38}
\providecommand{\natexlab}[1]{#1}
\providecommand{\url}[1]{\texttt{#1}}
\expandafter\ifx\csname urlstyle\endcsname\relax
  \providecommand{\doi}[1]{doi: #1}\else
  \providecommand{\doi}{doi: \begingroup \urlstyle{rm}\Url}\fi

\bibitem[Azar et~al.(2018)Azar, Schmalz, and Tecu]{azar18}
Jose Azar, Martin~C. Schmalz, and Isabel Tecu.
\newblock Anticompetitive effects of common ownership.
\newblock \emph{The Journal of Finance}, 73\penalty0 (4):\penalty0 1513--1565,
  2018.

\bibitem[Ball(2020)]{Ball2020}
Ian Ball.
\newblock Scoring strategic agents.
\newblock \emph{ArXiv:1909.01888}, 2020.

\bibitem[Baye and Kovenock(2008)]{Baye2008-bertrand}
Michael~R. Baye and Dan Kovenock.
\newblock \emph{Bertrand competition}.
\newblock Palgrave Macmillan UK, London, 2008.

\bibitem[Bergemann and Bonatti(2022)]{bergemann2022data}
Dirk Bergemann and Alessandro Bonatti.
\newblock Data, competition, and digital platforms.
\newblock 2022.

\bibitem[Bornier(1992)]{bornier92}
Jean Magnan~de Bornier.
\newblock {The ``Cournot-Bertrand Debate'': A Historical Perspective}.
\newblock \emph{History of Political Economy}, 24\penalty0 (3):\penalty0
  623--656, 09 1992.

\bibitem[Brown et~al.(2022)Brown, Hod, and Kalemaj]{brown2022performative}
Gavin Brown, Shlomi Hod, and Iden Kalemaj.
\newblock Performative prediction in a stateful world.
\newblock In \emph{International Conference on Artificial Intelligence and
  Statistics (AISTATS)}, pages 6045--6061. PMLR, 2022.

\bibitem[Br\"{u}ckner et~al.(2012)Br\"{u}ckner, Kanzow, and
  Scheffer]{bruckner12pred}
Michael Br\"{u}ckner, Christian Kanzow, and Tobias Scheffer.
\newblock Static prediction games for adversarial learning problems.
\newblock \emph{JMLR}, 13\penalty0 (1):\penalty0 2617--2654, September 2012.

\bibitem[Calvano and Polo(2021)]{calvano2021market}
Emilio Calvano and Michele Polo.
\newblock Market power, competition and innovation in digital markets: A
  survey.
\newblock \emph{Information Economics and Policy}, 54:\penalty0 100853, 2021.

\bibitem[Cr{\'e}mer et~al.(2019)Cr{\'e}mer, de~Montjoye, and
  Schweitzer]{cremer2019competition}
Jacques Cr{\'e}mer, Yves-Alexandre de~Montjoye, and Heike Schweitzer.
\newblock \emph{Competition Policy for the digital era : Final report}.
\newblock Publications Office of the European Union, 2019.

\bibitem[Dong and Ratliff(2021)]{dong2021approximate}
Roy Dong and Lillian~J Ratliff.
\newblock Approximate regions of attraction in learning with decision-dependent
  distributions.
\newblock \emph{arXiv preprint arXiv:2107.00055}, 2021.

\bibitem[Drusvyatskiy and Xiao(2022)]{drusvyatskiy2022stochastic}
Dmitriy Drusvyatskiy and Lin Xiao.
\newblock Stochastic optimization with decision-dependent distributions.
\newblock \emph{Mathematics of Operations Research}, 2022.

\bibitem[Fogg(2002)]{fogg02}
B.~J. Fogg.
\newblock Persuasive technology: Using computers to change what we think and
  do.
\newblock \emph{Ubiquity}, dec 2002.
\newblock \doi{10.1145/764008.763957}.

\bibitem[Frankel and Kartik(2019)]{FK19}
Alex Frankel and Navin Kartik.
\newblock {Muddled Information}.
\newblock \emph{Journal of Political Economy}, 127\penalty0 (4):\penalty0
  1739--1776, 2019.

\bibitem[Frankel and Kartik(2022)]{frankel2022improving}
Alex Frankel and Navin Kartik.
\newblock Improving information from manipulable data.
\newblock \emph{Journal of the European Economic Association}, 20\penalty0
  (1):\penalty0 79--115, 2022.

\bibitem[Ginart et~al.(2021)Ginart, Zhang, Kwon, and Zou]{GZKZ21}
Tony Ginart, Eva Zhang, Yongchan Kwon, and James Zou.
\newblock Competing {AI:} how does competition feedback affect machine
  learning?
\newblock In \emph{International Conference on Artificial Intelligence and
  Statistics (AISTATS)}, volume 130, pages 1693--1701. {PMLR}, 2021.

\bibitem[Hardt et~al.(2016)Hardt, Megiddo, Papadimitriou, and
  Wootters]{hardt16strat}
Moritz Hardt, Nimrod Megiddo, Christos Papadimitriou, and Mary Wootters.
\newblock Strategic classification.
\newblock In \emph{Innovations in Theoretical Computer Science (ITCS)}, 2016.

\bibitem[Hennessy and Goodhart(2020)]{HG20}
Christopher Hennessy and Charles Goodhart.
\newblock Goodhart's law and machine learning.
\newblock \emph{SSRN}, 2020.

\bibitem[Imbens and Rubin(2015)]{imbens2015causal}
Guido~W Imbens and Donald~B Rubin.
\newblock \emph{Causal inference in statistics, social, and biomedical
  sciences}.
\newblock Cambridge University Press, 2015.

\bibitem[Izzo et~al.(2021)Izzo, Ying, and Zou]{izzo2021learn}
Zachary Izzo, Lexing Ying, and James Zou.
\newblock How to learn when data reacts to your model: Performative gradient
  descent.
\newblock In \emph{International Conference on Machine Learning (ICML)}, 2021.

\bibitem[Jagadeesan et~al.(2022)Jagadeesan, Zrnic, and
  Mendler{-}D{\"{u}}nner]{JZM22}
Meena Jagadeesan, Tijana Zrnic, and Celestine Mendler{-}D{\"{u}}nner.
\newblock Regret minimization with performative feedback.
\newblock In \emph{International Conference on Machine Learning (ICML)}. PMLR,
  2022.

\bibitem[Lerner(1934)]{lerner34lernerindex}
A.~P. Lerner.
\newblock The concept of monopoly and the measurement of monopoly power.
\newblock \emph{The Review of Economic Studies}, 1\penalty0 (3):\penalty0
  157--175, 1934.

\bibitem[Li and Wai(2021)]{li2021state}
Qiang Li and Hoi-To Wai.
\newblock State dependent performative prediction with stochastic
  approximation.
\newblock \emph{ArXiv:2110.00800}, 2021.

\bibitem[Mendler-D\"{u}nner et~al.(2020)Mendler-D\"{u}nner, Perdomo, Zrnic, and
  Hardt]{mendler20stochasticPP}
Celestine Mendler-D\"{u}nner, Juan Perdomo, Tijana Zrnic, and Moritz Hardt.
\newblock Stochastic optimization for performative prediction.
\newblock In \emph{Advances in Neural Information Processing Systems},
  volume~33, pages 4929--4939, 2020.

\bibitem[Miller et~al.(2021)Miller, Perdomo, and Zrnic]{miller2021outside}
John~P Miller, Juan~C Perdomo, and Tijana Zrnic.
\newblock Outside the echo chamber: Optimizing the performative risk.
\newblock In \emph{International Conference on Machine Learning (ICML)}, pages
  7710--7720. PMLR, 2021.

\bibitem[Milli et~al.(2019)Milli, Miller, Dragan, and Hardt]{MMDH19}
Smitha Milli, John Miller, Anca~D. Dragan, and Moritz Hardt.
\newblock The social cost of strategic classification.
\newblock In \emph{Fairness, Accountability, and Transparency (FAccT)}, pages
  230--239, 2019.

\bibitem[Narang et~al.(2022)Narang, Faulkner, Drusvyatskiy, Fazel, and
  Ratliff]{narang2022multiplayer}
Adhyyan Narang, Evan Faulkner, Dmitriy Drusvyatskiy, Maryam Fazel, and
  Lillian~J. Ratliff.
\newblock Multiplayer performative prediction: Learning in decision-dependent
  games.
\newblock \emph{ArXiv:2201.03398}, 2022.

\bibitem[Narayanan and Kalyanam(2015)]{NK15}
Sridhar Narayanan and Kirthi Kalyanam.
\newblock Position effects in search advertising and their moderators: A
  regression discontinuity approach.
\newblock \emph{Marketing Science}, 34\penalty0 (3):\penalty0 388–407, may
  2015.
\newblock ISSN 1526-548X.

\bibitem[Parker et~al.(2020)Parker, Petropoulos, and
  Van~Alstyne]{parker2020digital}
Geoffrey Parker, Georgios Petropoulos, and Marshall~W Van~Alstyne.
\newblock Digital platforms and antitrust.
\newblock 2020.

\bibitem[Perdomo et~al.(2020)Perdomo, Zrnic, Mendler-D\"{u}nner, and
  Hardt]{PZMH20}
Juan~C. Perdomo, Tijana Zrnic, Celestine Mendler-D\"{u}nner, and Moritz Hardt.
\newblock Performative prediction.
\newblock In \emph{International Conference on Machine Learning (ICML)}, volume
  119, pages 7599--7609. PMLR, 2020.

\bibitem[Piliouras and Yu(2022)]{piliouras22}
Georgios Piliouras and Fang{-}Yi Yu.
\newblock Multi-agent performative prediction: From global stability and
  optimality to chaos.
\newblock \emph{Arxiv:2201.10483}, 2022.

\bibitem[Ray et~al.(2022)Ray, Ratliff, Drusvyatskiy, and Fazel]{raydecision}
Mitas Ray, Lillian~J Ratliff, Dmitriy Drusvyatskiy, and Maryam Fazel.
\newblock Decision-dependent risk minimization in geometrically decaying
  dynamic environments.
\newblock 2022.

\bibitem[Shmueli and Tafti(2020)]{shmueli20}
Galit Shmueli and Ali Tafti.
\newblock "{I}mproving" prediction of human behavior using behavior
  modification.
\newblock \emph{Arxiv:2008.12138}, 2020.

\bibitem[{Stigler Committee}(2019)]{stigler19}
{Stigler Committee}.
\newblock Final report: Stigler committee on digital platforms.
\newblock available at
  \href{https://research.chicagobooth.edu/stigler/media/news/committee-on-digitalplatforms-final-report}{https://research.chicagobooth.edu/stigler/media/news/committee-on-digitalplatforms-final-report},
  September 2019.

\bibitem[Syverson(2019)]{syverson19mp}
Chad Syverson.
\newblock Macroeconomics and market power: Context, implications, and open
  questions.
\newblock \emph{Journal of Economic Perspectives}, 33\penalty0 (3):\penalty0
  23--43, August 2019.

\bibitem[Thaler and Sunstein(2008)]{thaler08nudge}
Richard~H. Thaler and Cass~R. Sunstein.
\newblock \emph{Nudge: Improving Decisions about Health, Wealth, and
  Happiness}.
\newblock Yale University Press, 2008.

\bibitem[Ursu(2018)]{ursu2018}
Raluca~M. Ursu.
\newblock The power of rankings: Quantifying the effect of rankings on online
  consumer search and purchase decisions.
\newblock \emph{Marketing Science}, 37\penalty0 (4):\penalty0 530--552, 2018.

\bibitem[Ward(2022)]{ward22}
Jacob Ward.
\newblock \emph{The Loop: How Technology is Creating a World without Choices
  and How to Fight Back}.
\newblock Hachette Books, 2022.

\bibitem[Wood et~al.(2022)Wood, Bianchin, and Dall'Anese]{wood22}
Killian Wood, Gianluca Bianchin, and Emiliano Dall'Anese.
\newblock Online projected gradient descent for stochastic optimization with
  decision-dependent distributions.
\newblock \emph{IEEE Control Systems Letters}, 6:\penalty0 1646--1651, 2022.

\end{thebibliography}
\bibliographystyle{plainnat}

\appendix
\newpage

\section{Additional discussion}

\label{sec:discussion-app}

\subsection{Comparison with measures of market power in economics}
Traditional measures of market power in economic theory are based on classical markets of homogeneous goods, where a firm's primary action is choosing a price to sell the good or the quantity of the good to sell. The scalar nature of these quantities enables them to be easily compared across different market contexts and firms. In addition, the utility of the firm and the utility of participants are inversely related: a higher price yields greater utility for the firm and lower utility for all participants. This simple relationship enables directly reasoning about participant welfare and profit of firms. However, a digital economy is much more complex~\citep{stigler19, cremer2019competition} and classical measures can struggle to accurately characterize these economies. As an example, consider the following two textbook definitions of market power.

\begin{itemize}
    \item \textbf{Lerner index.} The Lerner index \citep{lerner34lernerindex} quantifies the pricing power of a firm, measuring by how much the firm can raise the price above marginal costs. Marginal costs reflect the price that would arise in a perfectly competitive market. A major issue of applying this standard definition of market power to digital economies is that it is not clear what the competitive reference state should look like, ``We have lost the competitive benchmark,''\footnote{Opening statement at the 2019 Antitrust and competition conference -- digital platforms, markets, and democracy} as Jacques Crémer said. Thus, measures based on profit margin cannot directly be adopted as a proxy for market power in digital economies. 
\item\textbf{Market share.} Measures such as the Herfindahl–Hirschman index (HHI), which is used by the US federal trade commission\footnote{See \url{https://www.justice.gov/atr/herfindahl-hirschman-index} (retrieved January, 2022).} to measure market competitiveness, are based on \textit{market share}: the fraction of participants who participate in a given firm. 
However, the validity of market share as a proxy for power relies on a specific model of competition where the elasticity of demand is low. This model is challenging to justify\footnote{This critique is similar to the disconnect between the Cournot model and the Bertrand model in classical economics \citep{bornier92}. E.g., ``concentration is worse than just a noisy barometer of market power'' \citep{syverson19mp}.} in the context of digital economies where opening an account on a platform is very simple and usually free of charge. In addition, not all participants with accounts on a digital platform are equally active and inactive participants should not factor into the market power of a firm in the same way as active participants. Market share is not sufficiently expressive to make this distinction.
\end{itemize}

In contrast, performative power is a causal notion of influence that does not require a precise specification of the market but is still sensitive to the nuances of the market. The definition therefore could serve as a useful tool in markets that resist a clean mathematical specification. 

\paragraph{Behavioral aspects.} Complex consumer behavior that plays a critical role in digital marketplaces. As outlined by the \citet{stigler19}, ``the findings from behavioral economics demonstrate an under-recognized market power held by incumbent digital platforms.'' In particular, behavioral aspects of consumers---such as tendencies for single-homing, vulnerability to addiction, and the impact of framing and nudging on participant behavior~\citep[e.g.][]{thaler08nudge, fogg02}---can be exploited by firms in digital economies, but do not factor into traditional measures of market power. By focusing on changes in participant features, performative power has the potential to capture the effects of these behavioral patterns while again sidestepping the challenges of explicitly modeling them. 

\subsection{From performative power to consumer harm}

Performative power focuses on measuring power rather than harm. The relationship to harm depends on the choice and interpretation of the outcome variable and requires additional substantive arguments. In  general, this connection can be achieved if the attributes $z(u)$ consists of the sensitive features that are impacted by the firm, the distance function is aligned with the utility function of participants, and the set~$\cF$ reflects actions that are taken by the firm. We implement this strategy to establish an exact correspondence between performative power and harm in the strategic classification setup.

\paragraph{Relating performative power and user burden in strategic classification.}
\label{appendix:optimizationSC}
The fact that a monopoly firm has nonzero performative power has consequences for the optimization strategies that it would use, as we discussed in Section \ref{sec:optimization}. To make this explicit, let's contrast the solutions of ex-ante and ex-post optimization in a simple one-dimensional setting.
\begin{example}[1-dimensional setting]
\label{example:1d}
Consider a 1-dimensional feature vector $x\in\R$ and suppose that the posterior $p(x) = \Pr[Y=1\mid X = x]$ is strictly increasing in $x$ with $\lim_{x \rightarrow -\infty} p(x) = 0$, and $\lim_{x \rightarrow \infty} p(x) = 1$. Now consider a set of actions $\cF$ that corresponds to the set of all threshold functions and set $\distance(x,x') = c(x,x') = |x - x'|$. Let $\thetaSL$ be the supervised learning threshold from ex-ante optimization, which is the unique value where $p(\thetaSL) = 0.5$. Then, the ex-post threshold lies at $\thetaPO=\theta_{\text{SL}} + \Delta \gamma$. 
\end{example}

In Example \ref{example:1d} ex-post optimization leads to a higher acceptance threshold than ex-ante optimization. Thus, for any setting where the participants utility is decreasing in the threshold (e.g., the class of utility functions that~\citet{MMDH19} call \emph{outcome monotonic}), this implies that ex-post optimization creates stronger negative externalities for participants than ex-ante optimization. Furthermore, the effect grows with the performative power of the firm. In the extreme case of the monopoly setting with no outside options, ex-post optimization can leave certain participants with a net utility of $0$ and thus can transfer the entire utility from these participants to the firm. 

\subsection{Monopoly power in heterogeneous setting}
\label{app:ex-heter}

Different participants are typically impacted differently by a classifier, depending on their relative position to the decision boundary, as visualized in Figure~\ref{fig:shift}. As a result of this heterogeneity, the upper bound in \eqref{eq:boundSC} is not necessarily tight, because the firm can not extract the full utility from all participants simultaneously. 

We investigate the effect of heterogeneity in a concrete 1-dimensional setting where $\distance_X(x,x') = c(x,x') = |x - x'|$. Consider a set of actions $\cF$ that corresponds to the set of all threshold functions. Suppose that the posterior $p(x) = \Pr[Y=1\mid X = x]$ satisfies the following regularity assumptions: $p(x)$ is strictly increasing in $x$ with $\lim_{x \rightarrow -\infty} p(x) = 0$, and $\lim_{x \rightarrow \infty} p(x) = 1$. Now, let $\thetaSL$ be the supervised learning threshold, which is the unique value where $p(\thetaSL) = 0.5$. We can then obtain the following bound on the performative power $\PP$ with respect to any $\cF$ assuming the firm's classifier is $\theta_{\text{SL}}$ in the current economy (see Proposition \ref{prop:1d}): 
\begin{equation}
0.5 \Delta\gamma \Pr_{\DBase} \big[x \in [\theta_{\text{SL}}, \theta_{\text{SL}} + 0.5 \Delta\gamma]\big] \le \PP \le \Delta\gamma\,.
\end{equation}
\noindent This bound illustrates how performative power in strategic classification depends on the fraction of participants that fall in between the old and the new threshold. As long as the density in this region is non-zero, a platform that offers $\Delta\gamma>0$ utility will also have strictly positive performative power, providing a lower bound on $\PP$. 

\begin{proposition}
\label{prop:1d}
Suppose that $\distance(x,x') = c(x,x') = |x - x'|$. Consider a set of actions $\cF$ that corresponds to the set of all threshold functions. Suppose that the posterior $p(x) = \Pr[Y=1\mid X = x]$ satisfies the following regularity assumptions: $p(x)$ is strictly increasing in $x$ with $\lim_{x \rightarrow -\infty} p(x) = 0$, and $\lim_{x \rightarrow \infty} p(x) = 1$. Now, let $\thetaSL$ be the supervised learning threshold, which is the unique value where $p(\thetaSL) = 0.5$. If the firm's classifier is $\theta_{\text{SL}}$ in the current economy, then performative power $\PP$ with $\cF$ can be bounded as:
\begin{equation}
0.5 \gamma \Pr_{\DBase} \big[x \in [\theta_{\mathrm{SL}}, \theta_{\mathrm{SL}} + 0.5 \Delta\gamma]\big] \le \PP \le 2 \Delta\gamma\,.
\end{equation}
\end{proposition}

\subsection{Background on the search advertisement study}

We provide additional context on the search advertisement study by \citet{NK15} on which we build the establish a lower-bound on performative power. They examine position effects in search advertising, where advertisements are displayed across a number of ordered slots whenever a keyword is searched. They show that the position effect of display  slot 1 versus display  slot 2 is $0.0048$ clicks per impression (see Table 2 in their manuscript).

To arrive at this number, the authors implemented a regression discontinuity approach to estimate the position effect. The input is a sample of data $(k, p, z, y)$ where $k$ is a keyword, $p \in \left\{1, 2\right\}$ is the position of the advertisement in the list of displayed content, $z$ is the AdRank score, and $y$ is the click-through-rate (CTR). The following local linear regression estimator 
\begin{equation}
\label{eq:locallinear}
 y = \alpha + \xi \mathrm I[p = 1] + \gamma_1 z + \gamma_2 z  \mathrm I[p = 1] + g(k)
\end{equation}
is applied to a subset of the data within an appropriate window size $\lambda>0$ around the threshold for fitting $\alpha,\xi,\gamma_1,\gamma_2,g$.

We are interested in the value $\xi$ which is an estimate of the \textit{position effect} of the display slot. To connect the causal effect estimate $\xi$ to the causal effect $\beta$ as in Definition~\ref{def:display-effect} we treat each incoming keyword query as a distinct ``viewer''. Following the query, the viewer $u$ either clicks on the advertisement in one of the display slots or does not click on any advertisement. The value $z_s(u)[i]$ corresponds to the probability that item~$i$ is consumed by viewer~$u$ under the scoring rule~$s$.  For $i=0$, the value $z_s(u)[0]$ corresponds to the probability that the viewer does not click on any advertisement. If item~$i$ is displayed, $z_s(u)[i]$ corresponds to the click-through-rate.  Hence $\beta=\gamma$ and $P\geq \gamma$.

\section{Proofs}

\subsection{Auxiliary results}
The proofs for Section \ref{sec:optimization} leverage the following lemma, which bounds the diameter of $\Theta$ with respect to Wasserstein distance in distribution map.
\begin{lemma}
\label{lemma:dist}
Let $\PP$ be the performative power with respect to $\Theta$. 
For any $\theta, \theta' \in \Theta$, it holds that $\cW(\cD(\theta), \cD(\theta')) \le 2 \PP$. 
\end{lemma}
\begin{proof}
Let $\theta_{\text{curr}}$ be the current classifier weights. We use the fact that for any weights $\theta'' \in \Theta$, it holds that $\cW(\cD(\theta_{\text{curr}}), \cD(\theta'')) \le \frac{1}{|U|} \sum_{u \in \cU} \E[\distance(z(u), z_{\theta''}(u))]$ where the expectation is over randomness in the potential outcomes. This follows from the definition of Wasserstein distance---in particular that we can instantiate the mass-moving function by mapping each participant to themselves. Thus, we see that:
\begin{align*}
   \cW(\cD(\theta), \cD(\theta')) &\le \cW(\cD(\theta), \cD(\theta_{\text{curr}})) + \cW(\cD(\theta_{\text{curr}}), \cD(\theta')) \\
   &\le \frac{1}{|U|} \sum_{u \in \cU} \E[\distance(z(u), z_{\theta}(u))] + \frac{1}{|U|} \sum_{u \in \cU} \E[\distance(z(u), z_{\theta'}(u))] \\
     &\le 2 \sup_{\theta'' \in \Theta} \frac{1}{|U|} \sum_{u \in \cU} \E[\distance(z(u), z_{\theta''}(u))] \\
   &\le 2\PP,
\end{align*}
where the last line uses the definition of performative power that bounds the effect of any $\theta$ in the action set $\Theta$ on the participant data $z$. 
\end{proof}

\subsection{Proof of Proposition \ref{prop:SL}}

Let $\phi$ be the previous deployment inducing the distribution on which the supervised learning threshold $\thetaSL$ is computed. Let $\theta^*$ be an optimizer of $\min_{\theta \in \Theta} \DPR(\thetaPO, \theta)$, where we recall the definition of the decoupled performative risk as $\DPR(\phi,\theta):=\E_{z\sim\cD(\phi)}\ell(\theta;\,z)$. Then, we see that for any $\phi$:
\begin{align*}
&\PR(\thetaSL)- \PR(\thetaPO) \\
&= \left(\DPR(\thetaSL, \thetaSL) - \DPR(\phi, \thetaSL) \right) + \DPR(\phi,\thetaSL) -  \DPR(\thetaPO, \thetaPO) \\
&\le \left(\DPR(\thetaSL, \thetaSL) - \DPR(\phi, \thetaSL) \right) + \DPR(\phi,\theta^*) -  \DPR(\thetaPO, \theta^*) \\
&\le  L_z \cW(\cD(\thetaSL), \cD(\phi)) +  L_z \cW(\cD(\phi), \cD(\thetaPO)) \\
&\le 4 L_z P\,.
\end{align*}
The first inequality follows because $\theta^*$ minimizes risk on the distribution $\cD(\thetaPO)$, while $\thetaSL$ minimizes risk on $\cD(\phi).$
The second inequality follows from the dual of the Wasserstein distance where $L_z$ is the Lipschitz constant of the loss function in the data argument~$z$.
The last inequality follows from Lemma \ref{lemma:dist}.

Now, suppose that $\ell$ is $\gamma$-strongly convex. Then we have that:
\[\DPR(\theta, \thetaPO) - \DPR(\theta, \thetaSL) \ge \frac{\gamma}{2} \|\thetaPO - \thetaSL\|^2
\]
Again applying Lemma \ref{lemma:dist},
\begin{align*}
    \PR(\thetaSL) &= \DPR(\thetaSL, \thetaSL) \\
    &\le \DPR(\theta, \thetaSL) + L_z \mathcal{W}(\cD(\phi), \cD(\thetaSL))\\
    &\le \DPR(\phi, \thetaSL) + 2 L_z P \\
     &\le  \DPR(\phi, \thetaPO) - \frac{\gamma}{2} \|\thetaPO - \thetaSL\|^2 + 2 L_zP \\
    &\le \DPR(\thetaPO, \thetaPO)  + L_z \cdot \cW(\cD(\phi), \cD(\thetaPO)) - \frac{\gamma}{2} \|\thetaPO - \thetaSL\|^2 + 2 L_z P \\ 
    &\le \PR(\thetaPO) +  4 L_z P - \frac{\gamma}{2} \|\thetaPO - \thetaSL\|^2.
\end{align*}
Using that $\PR(\thetaPO) \le \PR(\thetaSL)$, we find that
\[\frac{\gamma}{2} \|\thetaPO - \thetaSL\|^2 \le 4 L_z P\,.\]
Rearranging gives
\[\|\thetaPO - \thetaSL\| \le \sqrt{\frac{8 L_z P}{\gamma}}\,. \]

\subsection{Proof of Proposition \ref{prop:equilibrium}} 
Let's focus on firm $i$, fixing classifiers selected by the other firms. Let's take $\PR$ and $\DPR$ to be defined with respect to $\cD(\cdot) = \cD(\theta_1, \ldots, \theta_{i-1}, \cdot, \theta_{i+1}, \ldots, \theta_C)$. Let $\theta^* = \arg\min_{\theta} \DPR(\theta_i, \theta)$. We see that:
\begin{align*}
\PR(\theta^i) &\le 
\PR(\theta^*) \\
&\le \DPR(\theta_i, \theta^*) + L_z \mathcal{W}(\cD(\theta_i), \cD(\theta^*)) \\
&\le \min_{\theta} \DPR(\theta_i, \theta) + L_z \left(\frac{1}{|\cU|} \sum_{u \in \cU} \E[\distance(z(u), z_{\theta^*}(u))]\right)\\
&\le \min_{\theta} \DPR(\theta_i, \theta) + L_z P_i.
\end{align*}
Rewriting this, we see that:
\[\E_{z \sim \mathcal{D}} [\ell_i(\theta^i;\,z)] \le \min_{\theta} \E_{z \sim \mathcal{D}} [\ell_i(\theta;\,z)] + L_z P_i.  \]
If, in addition,  $\ell_i$ is $\gamma$-strongly convex, then we know that: 
\[L_z P_i \ge \E_{z \sim \mathcal{D}} [\ell_i(\theta^i;\,z)] - \min_{\theta} \E_{z \sim \mathcal{D}} [\ell(\theta;\,z)] \ge \frac{\gamma}{2} \|\theta^i -  \min_{\theta} \E_{z \sim \mathcal{D}} [\ell_i(\theta;\,z)]\|^2.  \]
Rearranging, we obtain that
\[\left\|\theta^i - \min_{\theta} \E_{z \sim \mathcal{D}} [\ell(\theta;\,z)]\right\|_2 \le \sqrt{\frac{2 L_z P_i}{\gamma}}\,.\qedhere\]

\subsection{Proof of Corollary \ref{cor:eq}}
Let $\PP$ be the performative power associated with the variables $z^{C=1}_{\theta}$. We first claim that the performative power of any firm in the mixture model is at most $P/C$. This follows from the fact that for a given firm the potential outcome $z_{\theta}(u)$ is equal to $z(u)$ with probability $1 - 1/C$ and equal to $z^{C=1}_{\theta}(u)$ with probability $1/C$.

Let's focus on platform $i$, fixing classifiers selected by the other platforms. 
Let's take $\PR$ and $\DPR$ to be with respect to $\cD(\cdot) = \cD(\theta^*, \cdots, \theta^*, \cdot, \theta^*, \ldots, \theta^*)$. Now, we can apply Proposition \ref{prop:equilibrium} to see that 
\[\PR(\theta^*) = \E_{z \sim \cD^{C=1}(\theta^*)} [\ell(\theta^*;\,z)] \le \min_{\theta} \E_{z \sim \cD^{C=1}(\theta^*)} [\ell(\theta;\,z)] + \frac{L_z P}{C}. \]
Thus, in the limit as $C \rightarrow \infty$, it holds that 
\[ \E_{z \sim \cD^{C=1}(\theta^*)} [\ell(\theta^*;\,z)] \rightarrow \min_{\theta} \E_{z \sim \cD^{C=1}(\theta^*)} [\ell(\theta;\,z)]\]
as desired.

\subsection{Proof of Lemma \ref{lemma:Pdiam}}

Fix a classifier~$f$ and a unit~$u$. By Assumption \ref{ass:behavior}, we know that $x(u)$ and $x_f(u)$ are both in $\cX_{\Delta\gamma}(u)$. 
The claim follows from
\[\mathrm{dist}(x_{\mathrm{orig}}(u), x_f(u))\le \sup_{x'\in \cX_{\Delta\gamma}(u)} \mathrm{dist}(x_{\mathrm{orig}}(u), x').\]


\subsection{Proof of Proposition~\ref{prop:personalized}}

The proof is by construction of a classifier $f^*:\R^m\rightarrow \{0,1\}$. 
For each individual $u$ we define the set 
\[\tilde \cX(u):=\argsup_{x'\in \cX_{\Delta\gamma}(u)} \mathrm{dist}(x(u), x').\] 
Now let $f^*$ be such that
\[f^*(x)=\begin{cases} 1 &x\in \tilde \cX(u) \text{ with } u=x[1] \\
0 &x\notin \tilde \cX(u) \text{ with } u=x[1]\end{cases}\]
where we used that the first coordinate of the feature vector $x$ uniquely identifies the individual. The effect of $f^*$ on a population $\cU$ corresponds to 
\begin{align}\frac 1 {|\cU|}\sum_{u\in\cU} \text{dist}(x(u),x_{f^*}(u)) = \frac 1 {|\cU|}\sum_{u\in\cU} \;\sup_{x'\in\mathcal X_{\Delta\gamma}(u)}\mathrm{dist}(x(u),x') 
\end{align}
Thus for any $\cF$ that contains $f^*$ the performative power is maximized.

\subsection{Proof of Corollary \ref{prop:BSC}}

Applying Lemma \ref{lemma:Pdiam}, it suffices to show that the diameter of the set $\mathcal{X}_{\Delta\gamma}(u)$ can be upper bounded by 
$2 L \Delta \gamma$ for any $u\in\cU$. We see that for any $x, x' \in \mathcal{X}_{\Delta\gamma}(u)$, it holds that:
\begin{align*}
    \distance(x, x') &\le L \cdot c(x, x') \\
    &\le L \cdot ( c(x_{\text{orig}}(u), x) + c(x_{\text{orig}}(u), x')) \\
    &\le 2 L \Delta \gamma,
\end{align*}
using that $c$ is a metric.

\subsection{Proof of Proposition \ref{cor:perfpower}}
\label{app:competition}

To prove this proposition, we show the following two intermediate results which are proven in the next two sections:
\begin{proposition}
\label{prop:competition}
Consider the 1-dimensional setup specified in Section~\ref{subsec:strategiccompete}, and suppose that the economy is at a symmetric state where both firms choose classifier $\theta$. For any $\cF$, consider one of the firms, let $\cF$ denote their action set and let $\theta_{\mathrm{min}}$ be the minimum threshold classifier in $\cF$.  Then, the performative power of the firm is upper bounded by: 
\[\PP \le L \min(c(\theta_{\mathrm{min}}, \theta), \gamma) + \gamma L  p_{\mathrm{reach}}([\theta_{\min},\theta]). \]
where  $p_{\mathrm{reach}}([\theta_{\min},\theta]):=\mathbb{P}_{\DBase} \left[x \in [\xi(\theta_{\mathrm{min}}), \xi(\theta)]\right]$ with $\xi(\theta')$ being the unique value such that $\xi(\theta') < \theta'$ and $c(\xi(\theta'), \theta') = 1$.
\end{proposition}

\begin{proposition}
\label{prop:zeroprofit}
Consider the 1-dimensional setup described in Section \ref{subsec:strategiccompete}. Then, a symmetric solution $[\theta^*, \theta^*]$ is an equilibrium if and only if $\theta^*$ satisfies 
\begin{equation}
\label{eq:zeroprofit}
  \textstyle
  \E_{(x,y) \sim \DBase}[y = 1 \mid x \ge \xi(\theta^*)] = \frac12,
\end{equation}
where $\xi(\theta^*)$ is the unique value such that $c(\xi(\theta^*), \theta^*) = \gamma$ and $\xi(\theta^*) < \theta^*$. Both firms earn zero utility at this equilibrium. Moreover, the set $\mathcal{F}^+(\theta^*)$ of actions that a firm can take at equilibrium that achieve nonnegative utility  is exactly equal to $[\theta^*, \infty)$, assuming the other firm chooses the classifier $\theta^*$. 
\end{proposition}

We now prove Proposition \ref{cor:perfpower} from these intermediate results. We apply Proposition \ref{prop:competition} to see that the performative power is upper bounded by 
\[B:=L \min(c(\theta_{\text{min}}, \theta^*), \gamma) + L \gamma p_{\mathrm{reach}}([\theta_{\min},\theta^*])\] 
where $(\theta^*, \theta^*)$ is a symmetric state. Using Proposition \ref{prop:equilibrium}, we see that $\mathcal{F}(\theta^*) = [\theta^*, \infty)$. This means that $\theta_{\text{min}} = \theta^*$, and so $\xi(\theta_{\text{min}}) =\xi(  \theta^*)$. Thus, $B=0$ which demonstrates that the performative power is upper bounded by $0$, and is thus equal to $0$.

\subsection{Proof of Proposition \ref{prop:competition}}

Consider a classifier $f \in \mathcal{F}(\theta)$ with threshold $\theta'$, and suppose that a firm changes their classifier to $f$. It suffices to show that: 
\[\frac{1}{|\cU|} \sum_{u \in \cU} \E[\distance(x(u), x_f(u))] \le L \min(c(\theta_{\text{min}}, \theta), \gamma) + L \gamma p_{\mathrm{reach}}([\theta_{\min},\theta]).  \]

For technical convenience, we reformulate this in terms of the cost function $c$. Based on the definition of $L$, it suffices to show that:
\[\frac{1}{|\cU|} \sum_{u \in \cU} \E[c(x(u), x_f(u))] \le \min(c(\theta_{\text{min}}, \theta), \gamma) + \gamma p_{\mathrm{reach}}([\theta_{\min},\theta]).  \]

\paragraph{Case 1: $\theta' > \theta$.} Participants either are indifferent between $\theta$ and $\theta'$ or prefer $\theta$ to $\theta'$.  Due to the tie breaking rule, the firm will thus lose all of its participants. Thus, all participants will switch to the other firm and adapt their features to that firm which has threshold $\theta$. This is the same behavior as these participants had in the current state, so $x_f(u) = x(u)$ for all participants $u$. This means that 
\[\frac{1}{|\cU|} \sum_{u \in \cU} \E[c(x(u), x_f(u))] = 0 \] as desired.

\paragraph{Case 2: $\theta' < \theta$.} Participants either are indifferent between $\theta$ and $\theta'$ or prefer $\theta'$ to $\theta$.  Due to the tie breaking rule, the firm will thus gain all of the participants. We break into several cases:
\[
\begin{cases}
x_f(u) = x(u) = x_{\text{orig}}(u) & \text{ if } x_{\text{orig}}(u) < \xi(\theta') \\
x_f(u) = \theta', x(u) = x_{\text{orig}}(u) & \text{ if } x_{\text{orig}}(u) \in [\xi(\theta'), \min(\theta', \xi(\theta)))] \\
x_f(u) = x(u) = x_{\text{orig}}(u)  & \text{ if }  x_{\text{orig}}(u) \in (\theta', \xi(\theta)) \\
x_f(u) = \theta', x(u) = \theta  & \text{ if }  x_{\text{orig}}(u) \in (\xi(\theta), \theta') \\
x_f(u) = x_{\text{orig}}(u), x(u) = \theta & \text{ if } x_{\text{orig}}(u) \in [\max(\theta', \xi(\theta)), \theta] \\
x_f(u) = x(u) = x_{\text{orig}}(u) & \text{ if }   x_{\text{orig}}(u) \ge \theta. 
\end{cases}
\]
The only cases that contribute to $\frac{1}{|\cU|} \sum_{u \in \cU} \E[c(x(u), x_f(u))]$ are the second, fourth, and fifth cases. Thus, we can upper bound $\frac{1}{|\cU|} \sum_{u \in \cU} \E[c(x(u), x_f(u))]$ by: 
\[\underbrace{\frac{1}{|\cU|} \sum_{u \in \cU \mid x_{\text{orig}}(u) \in [\xi(\theta'), \min(\theta', \xi(\theta))]}  \E[c(x(u), x_f(u))]}_{(A)} + \underbrace{\frac{1}{|\cU|} \sum_{u \in \cU \mid x_{\text{orig}}(u) \in (\xi(\theta), \theta')} \E[c(x(u), x_f(u))]}_{(B)}   \] 
\[+ \underbrace{\frac{1}{|\cU|} \sum_{u \in \cU \mid x_{\text{orig}}(u) \in [\max(\theta', \xi(\theta)), \theta]} \E[c(x(u), x_f(u))]}_{(C)}  \]
For (A), we see that 
\begin{align*}
  (A)
 &=  \frac{1}{|\cU|} \sum_{u \in \cU \mid x_{\text{orig}}(u) \in [\xi(\theta'), \min(\theta', \xi(\theta)))}  \E[c(x_{\text{orig}}(u), \theta')]   \\
  &\le  \frac{1}{|\cU|} \sum_{u \in \cU \mid x_{\text{orig}}(u) \in [\xi(\theta'), \min(\theta', \xi(\theta)))}  \E[c(\xi(\theta'), \theta')]   \\
   &= \gamma \cdot \mathbb{P}_{\DBase} [x \in [\xi(\theta'), \min(\theta', \xi(\theta))))] \\
 &\le \gamma \cdot \mathbb{P}_{\DBase} [x \in [\xi(\theta'), \xi(\theta))].
\end{align*}
For (B), we see that:
\begin{align*}
 (B)
 &=  \frac{1}{|\cU|} \sum_{u \in \cU \mid x_{\text{orig}}(u) \in (\xi(\theta), \theta')}  \E[c(\theta, \theta')]   \\
  &=  c(\theta, \theta') \cdot \mathbb{P}_{\DBase}[x \in (\xi(\theta), \theta')] \\
  &= \min(c(\theta, \theta'), \gamma) \cdot \mathbb{P}_{\DBase}[x \in (\xi(\theta), \theta')]. 
\end{align*}

For (C), we see that:
\begin{align*}
 (C)
 &=  \frac{1}{|\cU|} \sum_{u \in \cU \mid x_{\text{orig}}(u) \in [\max(\theta', \xi(\theta)), \theta]} \E[c(x_{\text{orig}}(u), \theta)]   \\
  &\le  \frac{1}{|\cU|} \sum_{u \in \cU \mid x_{\text{orig}}(u) \in [\max(\theta', \xi(\theta)), \theta]} \E[\min\left(c(\theta', \theta), c(\xi(\theta), \theta)  \right)]   \\
    &=  \frac{1}{|\cU|} \sum_{u \in \cU \mid x_{\text{orig}}(u) \in [\max(\theta', \xi(\theta)), \theta]} \E[\min\left(c(\theta', \theta), \gamma  \right)]   \\
 &= \min(c(\theta', \theta), \gamma) \cdot \mathbb{P}_{\DBase} [x \in [\max(\theta', \xi(\theta)), \theta]] 
\end{align*}
Putting this all together, we obtain that:
\[\frac{1}{|\cU|} \sum_{u \in \cU} \E[c(x(u), x_f(u))] \le  \gamma  p_{\mathrm{reach}}([\theta',\theta]) + \min(c(\theta', \theta), \gamma)\] for $p_{\mathrm{reach}}([\theta',\theta]):=\mathbb{P}_{\DBase} \left[x \in [\xi(\theta'), \xi(\theta)]\right]$ as desired.  
Since $\gamma  p_{\mathrm{reach}}([\theta',\theta]) + \min(c(\theta', \theta), \gamma)$ is decreasing in $\theta'$, this expression is maximized when $\theta' = \theta_{\text{min}}$. Thus we obtain an upper bound of
\[\gamma \cdot p_{\mathrm{reach}}([\theta_{\min},\theta]) + \min(c(\theta_{\text{min}}, \theta), \gamma)\,. \qedhere\]

\subsection{Proof of Proposition \ref{prop:zeroprofit}} 

The proof proceeds in two steps. First, we establish that $[\theta^*, \theta^*]$ is an equilibrium; next, we show that $[\theta, \theta]$ is not in equilibrium for $\theta \neq \theta^*$.

\paragraph{Establishing that $[\theta^*, \theta^*]$ is an equilibrium and $\mathcal{F}^+(\theta^*) = [\theta^*, \infty)$.}
First, we claim that $[\theta^*, \theta^*]$ is an equilibrium. At $[\theta^*, \theta^*]$, each participant chooses the first firm with $1/2$ probability. The expected utility earned by a firm is: 
\begin{align*}
  \frac{1}{2} \int_{\xi(\theta)}^{\infty} p_{\text{orig}}(x) (p(x) - (1-p(x)) \mathrm{d}x &=  \int_{\xi(\theta)}^{\infty} p_{\text{orig}}(x) (p(x) - 0.5) \mathrm{d}x \\
   &= \int_{\xi(\theta)}^{\infty} p_{\text{orig}}(x) p(x) \mathrm{d}x  - 0.5 \int_{\xi(\theta)}^{\infty} p_{\text{orig}}(x) \mathrm{d}x  \\
   &= \int_{\xi(\theta)}^{\infty} p_{\text{orig}}(x) \mathrm{d}x \left(\frac{\int_{\xi(\theta)}^{\infty} p_{\text{orig}}(x) p(x) \mathrm{d}x}{\int_{\xi(\theta)}^{\infty} p_{\text{orig}}(x) \mathrm{d}x} - \frac{1}{2}\right) \\
   &= \left( \int_{\xi(\theta)}^{\infty} p_{\text{orig}}(x) \mathrm{d}x \right) \left(\E_{(x,y) \sim \DBase} [y = 1 \mid x \ge \xi(\theta)] - \frac{1}{2} \right) \\
   &= 0\,. 
\end{align*}

If the firm chooses $\theta > \theta^*$, then since the cost function is strictly monotonic in its second argument, participants either are indifferent between $\theta$ and $\theta^*$ or prefer $\theta$ to $\theta^*$. Due to the tie breaking rule, the firm will thus lose all of its participants and incur $0$ utility. Thus the firm has no incentive to switch to $\theta$. 

If the firm chooses $\theta < \theta^*$, then it will gain all of the participants. The firm's utility will be:
\begin{align*}
  &\int_{\xi(\theta)}^{\infty} p_{\text{orig}}(x) (p(x) - (1-p(x)) \mathrm{d}x \\
&\qquad\qquad= \int_{\xi(\theta)}^{\xi(\theta^*)} p_{\text{orig}}(x) (p(x) - (1-p(x)) \mathrm{d}x + \int_{\xi(\theta^*)}^{\infty} p_{\text{orig}}(x) (p(x) - (1-p(x)) \mathrm{d}x  \\
  &\qquad\qquad= 2 \int_{\xi(\theta)}^{\xi(\theta^*)} p_{\text{orig}}(x) (p(x) - 0.5) \mathrm{d}x.
\end{align*}
It is not difficult to see that at $\theta^*$, it must hold that $p(\xi(\theta^*)) \le 0.5$. Since the posterior is strictly increasing, this means that $p(\xi(\theta)) < p(\xi(\theta^*)) = 0.5$, so the above expression is negative. This means that the firm will not switch to $\xi(\theta)$. 

Moreover, this establishes that $\mathcal{F}(\theta*) = [\theta^*, \infty)$. 

\paragraph{$[\theta, \theta]$ is not in equilibrium if $\xi(\theta^*)$ does not satisfy \eqref{eq:zeroprofit}.}
If $\theta < \theta^*$, then the firm earns utility 
\[\frac{1}{2} \left(\int_{\xi{\theta}}^{\infty} p_{\text{orig}}(x) (p(x) - (1-p(x))\right) \mathrm{d}x\,, \] 
which we already showed above was negative. Thus, the firm has incentive to change their threshold to above $\theta$ so that it loses the full participant base and gets $0$ utility. 

If $\theta > \theta^*$, then the firm earns  utility 
\[U = \frac{1}{2} \left(\int_{\xi(\theta)}^{\infty} p_{\text{orig}}(x) (p(x) - (1-p(x))\right) \mathrm{d}x\,, \] which is strictly positive. Fix $\epsilon > 0$, and suppose that the firm changes to a threshold $\theta'$ such that $c(\theta', \theta) = \epsilon$. Then it would gain all of the participants and earn utility: 
\begin{align*}
 \int_{\xi(\theta')}^{\infty} p_{\text{orig}}(x) (p(x) - (1-p(x)) \mathrm{d}x &= \int_{\xi(\theta')}^{\xi{\theta}} p_{\text{orig}}(x) (p(x) - (1-p(x)) \mathrm{d}x \\
 &\qquad\qquad+ \int_{\xi(\theta)}^{\infty} p_{\text{orig}}(x) (p(x) - (1-p(x)) \mathrm{d}x   \\
 &= \int_{\xi(\theta')}^{\xi(\theta)} p_{\text{orig}}(x) (p(x) - (1-p(x)) \mathrm{d}x + 2 U.
\end{align*}
We claim that this expression approaches $2U$ as $\epsilon \rightarrow 0$. To see this, note that $c(\xi(\theta'), \theta) \rightarrow \gamma$ and so $\xi(\theta')\rightarrow \xi(\theta)$ as $\epsilon \rightarrow 0$. This implies that $\int_{\xi(\theta')}^{\xi(\theta)} p_{\text{orig}}(x) (p(x) - (1-p(x)) \mathrm{d}x  \rightarrow 0$ as desired. Thus, the expression approaches $2U > U$ as desired. This means that there exists $\epsilon$ such that the firm changing to $\theta'$ results in a strict improvement in utility.

\subsection{ Proof of Theorem~\ref{thm:swapping}}
\label{app:ddd}
 Recall the definition of the action set $\mathcal S$. We prove Theorem~\ref{thm:swapping} by constructing a $s_\mathrm{swap}\in\mathcal S$ and relating the effect of a change in the score function from $s_\text{curr}$ to $s_\mathrm{swap}$ to the causal effect of position.

For $u\in\cU$ let $i_1(u)$ and $i_2(u)$ denote the index of the content item shown to user $u$ under $s_\text{curr}$ in the first and second display slot, respectively. 
Now, let the score function $s_\mathrm{swap}$ be such that the content items displayed in the first two display slots are swapped relative to $s_\text{curr}$, simultaneously for all users $u\in\cU$:  
\begin{equation}s_\mathrm{swap}(u)[i]=\begin{cases}s_\text{curr}(u)[i_2(u)] & i=i_1(u)\\ s_\text{curr}(u)[i_1(u)] & i=i_2(u) \\ s_\text{curr}(u)[i]& \text{otherwise}.\end{cases}\end{equation}

It holds that $s_\mathrm{swap}\in\mathcal S$, since $|s_\text{curr}(u)[i_1(u)]- s_\text{curr}(u)[i_2(u)]|\leq\delta$ for all $u\in\cU$.  We lower bound performative power as
\begin{align}\label{eq:swapping}
\PP = \sup_{s\in\mathcal S}\frac 1 {|\cU|}\sum_{u \in \cU} \E\left[\|z(u) - z_{{s}}(u)\|_1 \right]
\geq \frac 1 {|\cU|}\sum_{u \in \cU} \E\left[\|z(u) - z_{s_\mathrm{swap}}(u)\|_1 \right]
\end{align}

To bound the difference between the counterfactual variable $z(u)$ and $z_{s_\mathrm{swap}}(u)$, we decompose $s_\mathrm{swap}$ into a series of unilateral \textit{swapped} score functions, one for each viewer. 
The score function $s_\mathrm{swap}^u$ associated with viewer~$u$ swaps the scores of content that currently appears in the first two display slots for viewer $u$ and keeps the scores of the other viewers unchanged. 

Assumption~\ref{assumption:independence} implies that $z_{s_\mathrm{swap}}(u)=z_{s_\mathrm{swap}^u}(u)$, since there are no peer effects; $z_{s_\mathrm{swap}}(u)$ is independent of $s_\mathrm{swap}(u')$ for $u'\neq u$. Thus, we can aggregate the unilateral effects across all viewers $u\in\cU$ to obtain the effect of $s_\mathrm{swap}$ as:
\begin{align}
\PP &\geq \frac 1 {|\cU|}\sum_{u \in \cU} \E\left[\|z(u) - z_{s_\mathrm{swap}}(u)\|_1 \right]= \frac 1 {|\cU|}\sum_{u \in \cU} \E\left[\|z(u) - z_{s_\mathrm{swap}^u}(u)\|_1 \right].
\label{eq:uni}
\end{align}
Reasoning about unilateral effects allows us to relate the summands in \eqref{eq:uni} to the causal effect of position. In particular, focus on coordinate $i_1(u)$ in the norm, and let $Y_0(u)=z(u)[i_1(i)]$ and $Y_1(u)=z_{s_\mathrm{swap}^u}(u)[i_1(u)]$. Then, we have
\begin{align*}
\PP &\geq \frac 1 {|\cU|}\sum_{u \in \cU} {\E |z(u)[i_1] - z_{s_\mathrm{swap}^u}(u)[i_1]|} = \frac 1 {|\cU|}\sum_{u \in \cU} {\E |Y_0(u) - Y_1(u)|} = \beta.
\end{align*}
where the causal effect of position $\beta$ is defined as in Definition~\ref{def:display-effect}.

\subsection{Proof of Proposition~\ref{prop:1d}}

The upper bound follows from Corollary \ref{cor:perfpower}. For the lower bound, we take $f$ to be the threshold classifier given by $\thetaSL + \Delta \gamma$. We see that for $x_{\text{orig}}(u) \in [\theta_{\text{SL}}, \theta_{\text{SL}} + \Delta\gamma]$, it holds that $x_f(u) = \theta_{\text{SL}} + \Delta \gamma$ and $x(u) = x_{\text{orig}}(u)$. This means that the performative power is at least:
\begin{align*}
    \PP &=  \frac{1}{|\cU|} \sum_{u \in \cU} \E[\distance(x(u), x_f(u))] \\
    &= \frac{1}{|\cU|} \sum_{u \in \cU} \E[|x(u) - x_f(u)|] \\
    &\ge \frac{1}{|\cU|} \sum_{u \in \cU} I[x_{\text{orig}}(u) \in [\theta_{\text{SL}}, \theta_{\text{SL}} + \Delta\gamma]] \E[|\theta_{\text{SL}} + \Delta \gamma - x_{\text{orig}}(u)|] \\
    &\ge   \frac{1}{|\cU|} \sum_{u \in \cU} I[x_{\text{orig}}(u) \in [\theta_{\text{SL}}, \theta_{\text{SL}} + \frac12 \Delta\gamma]] \cdot \frac12 \Delta \gamma \\
    &\ge \frac12 \Delta \gamma \Pr_{\DBase}[x \in [\theta_{\text{SL}}, \theta_{\text{SL}} + \frac12 \Delta\gamma]],
\end{align*}
as desired. 
\end{document}